\documentclass[preprint]{elsarticle}

\usepackage{tikz}
\usetikzlibrary{intersections,calc,arrows.meta}

\usepackage[hang,small,bf]{caption}
\usepackage[subrefformat=parens]{subcaption}
\usepackage{ascmac}

\usepackage{pdfpages} 

\usepackage{bm}
\usepackage{amsmath}
\usepackage{mathrsfs}
\usepackage{amssymb}
\usepackage{booktabs}
\usepackage{comment}
\excludecomment{longcomment}
\usepackage{soul}
\usepackage{mathtools}
\usepackage{multirow}
\usepackage{graphicx}
\usepackage{pgfplots}
\pgfplotsset{compat=1.18}
\usepackage{adjustbox}

\numberwithin{equation}{section}

\usepackage{amsthm}

\newtheorem{thm}{Theorem}[section]
\newtheorem{lem}[thm]{Lemma}
\newdefinition{rem}{Remark}[section]

\newtheorem*{mainthm}{Main Theorem}
\newtheorem{prop}{Proposition}[section]

\theoremstyle{definition}
\newtheorem{dfn}{Definition}[section]
\theoremstyle{definition}

\theoremstyle{definition}
\newtheorem{ass}{Assumption}

\journal{Neural Networks}
\begin{document}

\begin{frontmatter}


  \title{Upper Bounds for Local Learning Coefficients of Three-Layer Neural Networks}
  \author{Yuki Kurumadani} 
  \ead{kurumadani@sigmath.es.osaka-u.ac.jp}
  \affiliation{organization={The University of Osaka},
            addressline={1 Chome-3 Machikaneyamacho}, 
            city={Toyonaka},
            postcode={Osaka 560-8531}, 
            country={Japan}}
\begin{abstract}
  Three-layer neural networks are known to form singular learning models, and their Bayesian asymptotic behavior is governed by the learning coefficient, or real log canonical threshold. 
  Although this quantity has been clarified for regular models and for some special singular models, broadly applicable methods for evaluating it in neural networks remain limited.
  
  Recently, a formula for the local learning coefficient of semiregular models was proposed, yielding an upper bound on the learning coefficient. 
  However, this formula applies only to nonsingular points in the set of realization parameters and cannot be used at singular points. 
  In particular, for three-layer neural networks, the resulting upper bound has been shown to differ substantially from learning coefficient values already known in some cases.

  In this paper, we derive a formula for an upper bound on local learning coefficients at a class of singular realization parameters in three-layer
neural networks.
  This formula can be interpreted as a counting rule under budget, demand, and supply constraints.
  In the non-polynomial real-analytic case, the formula applies in general settings,
  whereas in the polynomial case it applies under the restriction that the true distribution has no hidden units.
  In particular, our result covers activation functions such as the swish function
  and also includes polynomial activation functions under the above restriction,
  thereby extending previous results to a broader class of activation functions.

  We further show that, when the input dimension is one, the numerical value given by the right-hand side of our upper-bound formula agrees with the previously known learning coefficient, 
  thereby providing a useful comparison with known exact results.  
  Our result also provides a systematic perspective on how the weight parameters of three-layer neural networks affect the learning coefficient.

\end{abstract}  
  \begin{keyword}
    three-layer neural networks \sep singular learning theory \sep real log canonical threshold \sep algebraic geometry
  \end{keyword}
\end{frontmatter}

\section{Introduction}\label{sec_intro}
  Neural networks are important statistical models that are widely used in real data analysis.
  At the same time, they are singular learning models and exhibit properties that cannot be treated within classical theory.
  For example, in such models, the Fisher information matrix may fail to be positive definite, and therefore conventional information criteria such as AIC \cite{akaike1974} and BIC \cite{schwarz1978} cannot be directly applied to model selection.
  
  For such singular models, Watanabe \cite{watanabe2009} established a theory of learning coefficients using algebraic geometry, and Drton and Plummer \cite{drton2017} introduced the information criterion sBIC based on this theory.
  These studies have played an important role in model selection for singular models.
  
  Learning coefficients are not limited to model selection, but have also been used as tools for connecting singular learning theory with neural-network training dynamics and internal structure.
  Lau et al. \cite{Lau2025} introduced the Local Learning Coefficient (LLC) as a singularity-aware measure of effective complexity for deep neural networks, developed a scalable estimator, and empirically demonstrated that LLC estimates can capture differences in effective complexity associated with training heuristics.
  Hoogland et al. \cite{Hoogland2025} estimated the LLC throughout transformer training and showed that changes in loss landscape degeneracy, as quantified by the LLC, can be used to identify developmental stages, many of which coincide with interpretable changes in internal computational structure and input/output behavior.
  Wang et al. \cite{Wang2025} introduced refined Local Learning Coefficients, and used them to analyze how attention heads differentiate and specialize during training in transformer language models.

  Many of these studies rely on estimated local learning coefficients,
  and analytically available specific values are useful for interpreting and validating such estimates.
  However, determining the learning coefficient for a given statistical model is generally difficult and requires model-specific mathematical analysis.
  Therefore, it is also meaningful to derive explicit upper bounds.
  Such bounds provide theoretical indicators that bound the effective complexity of the model from above
  and give reference information for complexity assessment in model comparison and model selection,
  as well as for interpreting numerically estimated learning coefficients.
  
  For neural networks, Aoyagi and Watanabe \cite{Aoyagi2005} and Aoyagi \cite{Aoyagi2024} obtained learning coefficients for reduced-rank regression models and deep linear neural networks.
  For three-layer neural networks, Aoyagi \cite{Aoyagi2006,Aoyagi2009,Aoyagi2013,Aoyagi2019a,Aoyagi2019b} derived learning coefficients or their upper bounds in cases related to Vandermonde matrix type singularities, and Aoyagi \cite{Aoyagi2025} further obtained the learning coefficient for ReLU activation functions.
  Nevertheless, broadly applicable formulas for learning coefficients of three-layer neural networks
  with real-analytic activation functions beyond the special cases treated in previous studies
  remain largely unknown.
  
  Kurumadani \cite{kurumadani2025a,kurumadani2025b} defined a relatively simple class of singular models called semiregular models and derived a formula for the local learning coefficient at nonsingular points of the set of realization parameters.
  This result yields an upper bound on the learning coefficient of three-layer neural networks.
  On the other hand, there are cases in which this evaluation shows a large discrepancy from learning coefficient values that are already known \cite[Example 5.2]{kurumadani2025b}, suggesting that it is necessary to evaluate the local learning coefficient at singular points.

To address this issue, we derive a formula for an upper bound on the local learning coefficient
for three-layer neural networks with real-analytic activation functions.
The purpose of this paper is to give upper bounds for local learning coefficients
at singular realization parameters, where the previous formula for semiregular models
is not directly applicable.
The originality of the result lies in the fact that the upper bound is expressed
explicitly in terms of the Taylor expansion of the log-likelihood ratio function,
the rank of the Fisher information matrix, and the numbers of weight parameters
attached to redundant hidden units.

This formula covers a broader class of activation functions than those treated in previous work.
The non-polynomial real-analytic case includes, for example, the swish function,
whereas the polynomial case is treated under the restriction that the true distribution has no hidden units.
Moreover, the upper bound can be interpreted as an intuitive rule for the maximum number of items
that can be purchased under demand, supply, and budget constraints.
Furthermore, comparison with known local learning coefficients shows that,
when the number of units in the input layer is $1$, the numerical value given by the right-hand side
of the present upper-bound formula agrees with the local learning coefficient
obtained in the corresponding previously studied case.
On the other hand, when the number of units in the input layer is $2$ or more,
there are cases corresponding to reduced-rank regression in which the upper bound
does not coincide with the local learning coefficient,
showing that the inequality in the Main Theorem can be strict.

  This paper is organized as follows.
  In Section~\ref{sec:mainthm}, we present a formula for an upper bound on the local learning coefficient and apply it to three-layer neural networks.
  In Section~\ref{exsec01}, we illustrate applications of these results to three-layer neural networks and demonstrate consistency with previous research.
  In Section~\ref{sec_pf_mainthm}, we prove the main result, and in Section~\ref{sec_conclusion}, we conclude.

\subsubsection*{Notation and assumptions}
  Throughout this paper, we consider a statistical model $p(x|\theta)$ with a continuous parameter $\theta=(\theta_1,\ldots,\theta_d)\in \Theta(\subset \mathbb{R}^d)$ $(d\geq 1)$, and we denote the true distribution by $q(x)$.
  We assume that $q(x)>0$ for $q$-almost every $x$.
  We assume that the statistical model is \emph{realizable}, that is, there exists a parameter
  $\theta^*$ such that 
  $q(x)=p(x|\theta^*)$ for $q$-almost every $x$.
  We call such a parameter $\theta^*$ a \emph{realization parameter}, and we denote by
  $\Theta^*$ the set of all realization parameters.
  We assume that the prior density \(\varphi(\theta)\) has compact support in the parameter region \(\Theta\).
  Restricting \(\Theta\) to this support if necessary, we may assume that \(\Theta\) is compact.
  We denote by \(\Theta^*\) the set of realization parameters in this restricted region, and assume that \(\varphi(\theta^*)>0\) for every \(\theta^*\in\Theta^*\).
  Let $X$ be a random variable distributed according to the true distribution $q$, and let
  $\mathbb{E}_X[\cdot]$ denote expectation with respect to $X$.
  In this paper, we assume that expectation and partial differentiation with respect to \(\theta\) can be interchanged.

  We define the log-likelihood ratio function \(f\) and the Kullback--Leibler divergence \(K(\theta)\) by
  \[
  f(x|\theta):=\log \frac{q(x)}{p(x|\theta)},\qquad
  K(\theta):=\mathbb{E}_X\left[\log \frac{q(X)}{p(X|\theta)}\right].
  \]
  We assume that \(K(\theta)\) is continuous on \(\Theta\), and that, in a neighborhood of each
  \(\theta^*\in\Theta^*\), the map \(\theta\mapsto f(X|\theta)\) is an
  \(L^2(q)\)-valued analytic function.
  Then \(K(\theta)\) is also analytic in a neighborhood of each \(\theta^*\in\Theta^*\), and
  \(\Theta^*=K^{-1}(0)\) is compact.
  Furthermore, we define the Fisher information matrix at $\theta=\theta^*$ by
  \[
  I:=\mathrm{Cov}\left(\left.\frac{\partial f(X|\theta)}{\partial\theta_i}\right|_{\theta=\theta^*},
  \left.\frac{\partial f(X|\theta)}{\partial\theta_j}\right|_{\theta=\theta^*}\right)_{i,j=1,\ldots,d},
  \]
  and we denote the rank of this matrix by $r$.
  
  We say that random variables $X_1,\ldots,X_n$ are linearly independent if
  \[
  c_1,\ldots,c_n\in\mathbb{R},\quad\sum_{i=1}^n c_iX_i=0~q\text{-a.s.} ~ \Rightarrow\ c_1=\cdots=c_n=0.
  \]
  
  \begin{thm}[Resolution of singularities]\label{resolution thm}\cite{hironaka1964}\cite[Theorem 2.3]{watanabe2009}
    Let $F(x)$ be a real analytic function defined in a neighborhood of the origin in $\mathbb{R}^d$ that is not identically zero, and suppose that $F(0)=0$.
    Then one can find an open set $W\subset \mathbb{R}^d$ containing the origin,
    a real analytic manifold $U$, and a proper analytic map $g:U\to W$
    such that the following conditions hold:
    \begin{itemize}
      \item[(1)] Let $W_0:=F^{-1}(0)$ and $U_0:=g^{-1}(W_0)$. Then $g:U\setminus U_0\to W\setminus W_0$ is an analytic isomorphism.
      \item[(2)] For any point $Q\in U_0$, we can take local coordinates $u=(u_1,\ldots,u_d)$ on $U$ with origin at $Q$ such that
        \begin{align}\label{eq:normalcrossing}
          F(g(u))=a(u)u_1^{k_1}u_2^{k_2}\cdots u_d^{k_d},\\
          \left|g^{\prime}(u)\right|=\left|b(u)u_1^{h_1}u_2^{h_2}\cdots u_d^{h_d}\right|,\notag
        \end{align}
        where $k_i,h_i$ $(i=1,\ldots,d)$ are nonnegative integers, and $a(u),b(u)$ are real analytic functions satisfying $a(u)\neq 0$ and $b(u)\neq 0$.
    \end{itemize}
  \end{thm}
  We call a representation of the form \eqref{eq:normalcrossing} a \emph{normal crossing}.
  
  \begin{dfn}[learning coefficient]\label{LambdaDef}
    Let the Kullback--Leibler divergence $K(\theta)$ be a real analytic function defined on an open set $O\subset \mathbb{R}^d$ 
    satisfying $\Theta^*\subset O$.
    For each point $P\in \Theta^*$, 
    after translating coordinates so that $P$ becomes the origin in $\mathbb{R}^d$, we can apply Theorem~\ref{resolution thm}, and we fix one triple $(W,U,g)$ guaranteed by Theorem~\ref{resolution thm}(2).
    Moreover, in a neighborhood of any point $Q\in U_0$, we denote by $h_i^{(Q)}$ and $k_i^{(Q)}$ the nonnegative integers $h_i$ and $k_i$ given by Theorem~\ref{resolution thm}(2).
    \begin{itemize}
      \item[(1)]
        We define the \emph{local learning coefficient} at $P$, denoted by $\lambda_P$, and the \emph{learning coefficient} for the compact set $\Theta^*$, denoted by $\lambda$, by
        \[
        \lambda_P :=\inf_{Q\in U_0}\left\{\min_{i=1,\ldots,d}\frac{h_i^{(Q)}+1}{k_i^{(Q)}}\right\},\qquad
        \lambda   :=\inf_{P\in \Theta^*}\lambda_P,
        \]
        following \cite[Definition 2.7, Theorem 6.6, Definition 6.4]{watanabe2009}, \cite{Aoyagi2025, Lau2025}.
        If $k_i=0$, we define $(h_i+1)/k_i=\infty$.
      \item[(2)]
        In (1), for a point $P\in \Theta^*$ that attains the minimum, we define the \emph{multiplicity} $\mathfrak{m}$ as the maximum number of indices $i$ satisfying $\lambda_P=(h_i^{(Q)}+1)/k_i^{(Q)}$.
        (If there are multiple points $P\in\Theta^*$ that attain the minimum, then the multiplicity is defined as the maximum of these maximal numbers of indices.)
    \end{itemize}
  \end{dfn}

\section{Main Theorem}\label{sec:mainthm}
  In this section, we present a formula for an upper bound on the local learning coefficient for a class of singular models beyond three-layer neural networks,
and then apply it to three-layer neural networks.
Throughout this paper, an ideal means an ideal in the local ring of analytic functions at the point under consideration.

\begin{mainthm}
  Let $\alpha, \beta\in\mathbb{Z}_{\geq 1}$, 
  $\gamma\in\mathbb{Z}_{\geq 1}\cup\{\infty\}$, 
  $r\in\mathbb{Z}_{\geq 0}$, 
  and let \((m_s)_{1\leq s\leq \gamma}\) and 
  \((n_s)_{1\leq s\leq \gamma}\) be sequences of positive integers. 
  Assume that \((m_s)_{1\leq s\leq \gamma}\) is strictly increasing.
  Consider a statistical model $p(x|\theta,a,b)$ with $r+\alpha+\beta$ parameters $\theta=(\theta_1,\ldots,\theta_r)$, $a=(a_1,\ldots,a_\alpha)$, and $b=(b_1,\ldots,b_\beta)$.
  Assume that the origin $P:(\theta,a,b)=(0,0,0)$ is a realization parameter and is an interior point of $\Theta$.
  Suppose that the Taylor expansion of $f(X|\theta=0,a,b)$ can be written in the form
  \[
    f(X|\theta=0,a,b)
    =
    \sum_{s=1}^\gamma \sum_{n=1}^{n_s} g_{s,n}(a,b)Z_{s,n}
    +
    \sum_{s=1}^\infty h_s(a,b)W_s,
  \]
  where $g_{s,n}(a,b)$ and $h_s(a,b)$ are analytic functions, and
  $Z_{s,n}$ and $W_s$ are random variables.
  We assume that the lowest degree of $g_{s,n}(a,b)$ with respect to $b$ is $m_s$.
  Define $L$ and $n_s^\ast$ as follows.
  \footnote{When $\gamma=\infty$, note that $\sum_{s=1}^\gamma n_s\geq \alpha$ and
             $L:=\min\{l\in\mathbb{Z}_{\geq 1}\mid \sum_{s=1}^l n_s\geq \alpha\}$.}
  \begin{align*}
    L&:=
    \begin{cases}
      \gamma & \text{if } \sum_{s=1}^\gamma n_s< \alpha\\
      \min\{l=1,\ldots,\gamma\mid \sum_{s=1}^l n_s\geq \alpha\} & \text{otherwise}
    \end{cases}\\
    n_s^\ast&:=
    \begin{cases}
      n_s & s<L\\
      \min\{n_L, \alpha-\sum_{s=1}^{L-1}n_s\} & s=L
    \end{cases}.
  \end{align*}
  Let $\bar{g}_{s,n}$ denote the lowest-degree term of $g_{s,n}$ (that is, the homogeneous polynomial of degree $m_s$ with respect to $b$), and assume that the following conditions hold.
  \begin{itemize}
    \item [(i)] 
      $\forall s,\forall n,~~g_{s,n}(a=0,\forall b) = 0$.
    \item [(ii)] 
      There exists $b\neq0$ such that the rank of
      $\left.\frac{\partial \bar{g}}{\partial a}\right|_{a=0}$ is $\sum_{s=1}^L n_s^\ast$, where
      $\bar{g}:=(\bar{g}_{1,1},\ldots,\bar{g}_{1,n_1^\ast},\ldots,\bar{g}_{L,1},\ldots,\bar{g}_{L,n_L^\ast})$.
      Here, $\partial\bar{g}/\partial a$ denotes the Jacobian matrix of $\bar{g}$ with respect to $a=(a_1,\ldots,a_\alpha)$, and the rank is evaluated at the above value of $b$.
    \item [(iii)] 
      The random variables $(Z_{1,1},\ldots,Z_{1,n_1},\ldots,Z_{L,1},\ldots,Z_{L,n_L})$ are linearly independent, and,
      when $r>0$, the following random variables are linearly independent:
      \[
        \left.
          \frac{\partial f}{\partial \theta_1}
        \right|_{(\theta,a,b)=0}
        ,\ldots,
        \left.
          \frac{\partial f}{\partial \theta_r}
        \right|_{(\theta,a,b)=0}
        ,Z_{1,1},\ldots,Z_{1,n_1},\ldots,Z_{L,1},\ldots,Z_{L,n_L}.
      \]
    \item [(iv)] For any $s$, the analytic function $h_s(a,b)$ belongs to the ideal 
    $(
    g_{s,n}(a,b)\mid 1\leq s\leq \gamma,1\leq n\leq n_s
    )^2$.
  \end{itemize}
  Then the local learning coefficient $\lambda_P$ has the following upper bound:
  \begin{align}
    \lambda_P & \leq 
    \frac{r}{2} + 
    \begin{cases}
      \frac{\beta}{2m_1} & \text{if } K=0\\
      \frac{\beta+\sum_{s=1}^{K}(m_{K+1}-m_s)n_s^\ast}{2m_{K+1}}      
      \left(
      =\frac{\sum_{s=1}^{K}n_s^\ast}{2}+\frac{\beta-\sum_{s=1}^{K}m_sn_s^\ast}{2m_{K+1}}
      \right)
        & \text{if } 1\leq K\leq L-1\\
      \frac{\sum_{s=1}^L n_s^\ast}{2} & \text{if } K= L
    \end{cases},\label{eq_mainthm01}
  \end{align}
  where
  $K:=\max\left\{k=0,\ldots,L \middle| \sum_{s=1}^k m_sn_s^\ast\leq \beta\right\}$.
  The case $K=0$ means that $\beta<m_1n_1^\ast$ holds.
\end{mainthm}

  By the Main Theorem, the six quantities $(r,\alpha,\beta,\gamma,(m_s),(n_s))$ have a substantial impact on the local learning coefficient $\lambda_P$.
  The quantity $r$ is the rank of the Fisher information matrix at the point $P$, and $\alpha+\beta$ is the number of parameters for which the Fisher information matrix degenerates.
  In the Taylor expansion of $f$, the quantities $\gamma$, $m_s$, and $n_s$ represent, respectively,
  the number of distinct groups of terms with different lowest degrees in $b$,
  the lowest degree in $b$ of the $s$-th group,
  and the number of linearly independent random variables associated with the $s$-th group.

  Table~\ref{tab:mainthm_nn_correspondence}
  summarizes the concrete values of these quantities in the
  three-layer neural-network cases analyzed in Section~\ref{exsec01}.
  The network notation used in the table is formally introduced there:
  $N$, $H$, and $M$ denote the numbers of units in the input layer,
  the hidden layer, and the output layer, respectively;
  $H^*$ denotes the number of units in the hidden layer of the true distribution;
  and, in the linear case, $R:=\operatorname{rank}(A^*B^*)$.
  In the general polynomial case, $S$ denotes the number of nonzero polynomial degrees.
  The realization parameters $P_1$ and $P_2$ are defined in Section~\ref{exsec01}.

\begin{center}
\captionof{table}{Concrete values of the quantities in the Main Theorem
for the three-layer neural-network cases considered in Section~\ref{exsec01}}
\label{tab:mainthm_nn_correspondence}
\small
\setlength{\tabcolsep}{2.5pt}
\renewcommand{\arraystretch}{1.18}
\begin{adjustbox}{max width=\linewidth,center}
  \begin{tabular}{@{}lcccccc@{}}
    \toprule
    {\Large Case}
      & {\Large $r$}
      & {\Large $\alpha$}
      & {\Large $\beta$}
      & {\Large $\gamma$}
      & {\Large $m_s$}
      & {\Large $n_s$} \\
    \midrule

    \begin{tabular}[c]{@{}l@{}}
      non-polynomial,\\
      analytic, $P_1$
    \end{tabular}
      & $(M+N)H^*$
      & $M(H-H^*)$
      & $N(H-H^*)$
      & $\infty$
      & \begin{tabular}[c]{@{}c@{}}
          nonzero derivative\\
          orders of $\sigma$
        \end{tabular}
      & $M\binom{m_s+N-1}{m_s}$ \\

    \begin{tabular}[c]{@{}l@{}}
      non-polynomial,\\
      analytic, $N=1$, $P_2$
    \end{tabular}
      & $(M+1)H^*$
      & $M(H-H^*)$
      & $H-H^*$
      & $\infty$
      & $s$
      & $M-\mathbf{1}_{s=1}$ \\

    $\sigma(x)=x$
      & $R(M+N-R)$
      & $(M-R)(H-R)$
      & $(N-R)(H-R)$
      & $1$
      & $1$
      & $(M-R)(N-R)$ \\

    \begin{tabular}[c]{@{}l@{}}
      general polynomial,\\
      $H^*=0$
    \end{tabular}
      & $0$
      & $HM$
      & $HN$
      & $S$
      & \begin{tabular}[c]{@{}c@{}}
          nonzero polynomial\\
          degrees
        \end{tabular}
      & $M\binom{m_s+N-1}{m_s}$ \\

    \bottomrule
  \end{tabular}
\end{adjustbox}
\end{center}

  For example, in the first row,
  $r=(M+N)H^*$ is the number of parameters attached to the $H^*$ true hidden units,
  while $\alpha=M(H-H^*)$ and $\beta=N(H-H^*)$
  are the numbers of output-side and input-side weights attached to the
  $H-H^*$ redundant hidden units, respectively.
  Moreover,
  $n_s=M\binom{m_s+N-1}{m_s}$
  counts pairs consisting of an output coordinate and a multi-index
  $(h_1,\ldots,h_N)\in\mathbb{Z}_{\geq0}^N$
  satisfying $h_1+\cdots+h_N=m_s$. 
  In the second row, when \(M=1\), the term with \(s=1\) is absent, and the same formula is
  understood after reindexing the sequence as \(m_s=s+1\) and \(n_s=1\).

\begin{rem}\label{rem_mainthm01}
  The result \eqref{eq_mainthm01} of the Main Theorem for $2\lambda_P$ can be interpreted as a counting rule under budget and supply constraints.
  Suppose that a store displays items on $\gamma$ shelves, and that shelf $s$ contains $n_s$ items with unit price $m_s$, where $(m_s)$ is a strictly increasing sequence $(s=1,2,\ldots,\gamma)$.
  Thus, the total inventory is $\sum_{s=1}^\gamma n_s$ items.
  Suppose that the demand is $\alpha$ items, and that no more than this amount is purchased.
  Given budget $\beta$, consider purchasing as many items as possible starting from shelf $1$ in order.

  In this setting, the quantities $L$, $n_s^\ast$, and $K$ in the Main Theorem represent, respectively, the index of the shelf at which we finish collecting items when the budget $\beta$ is ignored, the number of items to be purchased from each shelf when the budget $\beta$ is ignored, and the index of the last shelf from which all items can be purchased when the budget $\beta$ is taken into account.

  Under these definitions, the second term in \eqref{eq_mainthm01} for $2\lambda_P$ represents the \emph{maximum total number of items purchased (allowing fractional quantities) under the constraints of the number of shelves $\gamma$, demand $\alpha$, budget $\beta$, prices $m_s$, and inventories $n_s$}.
  More specifically, the first case of \eqref{eq_mainthm01} corresponds to the situation in which the budget is exhausted while purchasing items on the first shelf.
  The second case corresponds to the situation in which all items up to shelf $K$ are purchased, and then items on shelf $K+1$ are purchased until the budget is exhausted.
  The third case corresponds to the situation in which all desired items can be purchased within the budget.
\\

  For two tuples
  \[
    (r^{(1)},\alpha^{(1)},\beta^{(1)},\gamma^{(1)},m_s^{(1)}, n_s^{(1)}),
    (r^{(2)},\alpha^{(2)},\beta^{(2)},\gamma^{(2)},m_s^{(2)}, n_s^{(2)}),
  \]
  let $\lambda^{(1)}$ and $\lambda^{(2)}$ denote the right-hand sides of \eqref{eq_mainthm01} for the respective tuples.
  Then the following relations hold.
  Since each of them can be verified by straightforward calculation, we omit the proofs.
  \begin{itemize}
    \item [(1)]
      If $(r,\alpha,\beta,\gamma,n_s)$ are the same for both tuples and $m_s^{(1)}\leq m_s^{(2)}$ holds for all $s$, then $\lambda^{(1)}\geq\lambda^{(2)}$ holds.
      This indicates that, as the unit prices increase, the number of items that can be purchased decreases accordingly.

    \item [(2)]
      If $(r,\alpha,\beta,\gamma,m_s)$ are the same for both tuples and $n_s^{(1)}\leq n_s^{(2)}$ holds for all $s$, then $\lambda^{(1)}\leq\lambda^{(2)}$ holds.
      This indicates that, as the inventories on all shelves increase, the number of items that can be purchased cheaply also increases.

    \item [(3)]
      For $(m_s)$ and $(n_s)$, inserting a shelf with $n_s=0$ at any position leaves $\lambda$ unchanged.
      This is because adding a shelf with no inventory does not affect the number of items obtained.
      Note that, in this case, the indices $L$ and $K$, which represent shelf numbers, change according to the number of inserted shelves.

  \end{itemize}
\end{rem}

\begin{rem}\label{rem_mainthm}
  As shown in the proof, after the first coordinate transformation, we consider
  the coordinate chart on which a normal crossing has not yet been obtained.
  On this chart, near an arbitrary point $b\neq0$ satisfying the rank condition in Condition~(ii), 
  the inverse function theorem gives the coordinate transformation \(a\mapsto a^\prime\).
  The subsequent coordinate transformations then give a normal crossing.

  Hence the Main Theorem gives an upper bound for \(\lambda_P\) in general.
  If Condition~(ii) holds for every \(b\neq0\), then no remaining coordinate
  neighborhood is left untreated, and \eqref{eq_mainthm01} becomes an equality.
  In this case, the multiplicity of \(\lambda_P\) is \(2\) if there exists
  \(k\ (1\leq k\leq L)\) such that
  $
    \beta=\sum_{s=1}^k m_s n_s^\ast,
  $
  and is \(1\) otherwise.

  If there exists \(b\neq0\) not satisfying Condition~(ii), then the present
  argument gives only the upper bound.  To conclude equality in that case, one
  would also have to analyze the local coordinate neighborhoods arising from such
  rank-degenerate points and show that they do not give smaller candidate values
  than the right-hand side of \eqref{eq_mainthm01}.
  
\end{rem}

\section{Upper bounds for local learning coefficients of three-layer neural networks}\label{exsec01}
In this section, we confirm that the Main Theorem is applicable to three-layer neural networks with $N(\geq 1)$ units in the input layer, $H(\geq 1)$ units in the hidden layer, and $M(\geq 1)$ units in the output layer.
For detailed arguments, see \ref{sec_app_mainthm}.

Let \(X:=(X_1,\ldots,X_N)^\top\) be an \(\mathbb{R}^N\)-valued random variable.
We assume that the distribution \(q_X\) of \(X\) has compact support.
Let \(\mathcal{N}\) be an \(M\)-dimensional standard normal random variable
independent of \(X\).
For an \(\mathbb{R}^M\)-valued random variable \(Y\), we consider a three-layer neural network with parameters
$A:=(a_{k,i})\in\mathbb{R}^{M\times H}$ and $B:=(b_{i,j})\in\mathbb{R}^{H\times N}$:
\[
  Y=A\sigma(BX)+\mathcal{N}.
\]
We assume that the activation function $\sigma$ is analytic and satisfies $\sigma(0)=0$.
In what follows, we denote the $i$-th column of $A$ by $a_{\cdot,i}\in\mathbb{R}^M$ and write
$b_i:=(b_{i,1},\ldots,b_{i,N})^\top\in\mathbb{R}^N$ for the transpose of the $i$-th row of $B$.
We write $\sigma(BX):=(\sigma(b_1^\top X),\ldots,\sigma(b_H^\top X))^\top$.

We denote by $H^*(0\leq H^*<H)$ the number of hidden-layer units in the true distribution, 
and assume that the true distribution can be represented by $A^*\in\mathbb{R}^{M\times H^*}$ and $B^*\in\mathbb{R}^{H^*\times N}$ as
\[
  Y=
  \begin{cases}
    A^*\sigma(B^*X)+\mathcal{N} & \text{if } H^*\geq 1,\\
    \mathcal{N} & \text{if } H^*=0.
  \end{cases}
\]
We may assume without loss of generality that every column vector of $A^*$ is nonzero.
Indeed, if all weights connecting a hidden-layer unit to the output layer are zero,
then removing that unit does not change the network output.
Thus this condition amounts to taking \(H^*\) to be minimal among
representations of the true distribution of this form.

We consider applying the Main Theorem to the following realization parameter
$P_1\in\Theta^*$:
\begin{align*}
  P_1:&
  \begin{cases}
    a_{\cdot,i}=a_{\cdot,i}^*~(1\le i\le H^*),\quad
    a_{\cdot,i}=0~(H^*+1\le i\le H),\\
    b_i=b_i^*~(1\le i\le H^*),\quad
    b_i=0~(H^*+1\le i\le H).
  \end{cases}
\end{align*}

\subsection{The case where $\sigma$ is a non-polynomial analytic function}\label{sec_ex_nonpolynomial}
Throughout this subsection, we assume that $\sigma$ is not a polynomial.
Let $(m_s)$ be the increasing sequence of integers $s$ such that $\sigma^{(s)}(0)\neq 0$.
Assume that, for every positive integer \(S\), the following family of random variables is linearly independent:
\begin{equation}\label{eq_sing01_01}
  \begin{array}{ll}
    \sigma(b_i^{*\top}X)
      & (1\leq i \leq H^*),\\
    \sigma^{\prime}(b_i^{*\top}X)X_j
      & (1\leq i \leq H^*,~ 1\leq j \leq N),\\
    \displaystyle\prod_{j=1}^N X_j^{h_j}
      & (1\leq s\leq S,~ h_1+\cdots+h_N=m_s,~ h_j\geq 0).
  \end{array}
\end{equation}
These correspond respectively to
$\left.\frac{\partial f}{\partial a_{k,i}}\right|_{\theta=\theta^*}$,
$\left.\frac{\partial f}{\partial b_{i,j}}\right|_{\theta=\theta^*}$, and
$\left.\frac{\partial^{m_s+1} f}{\partial a_{k,i}\partial b_{i,1}^{h_1}\cdots\partial b_{i,N}^{h_N}}\right|_{\theta=\theta^*}$
(see \ref{sec_app_mainthm}).
In this case, the Main Theorem applies with
$
(r,\alpha,\beta,\gamma,m_s,n_s)=((M+N)H^*, M(H-H^*), N(H-H^*),\infty, m_s, M\textstyle\binom{m_s+N-1}{m_s}),
$
and by using
\begin{align*}
  L:=&\min\left\{
    l\in\mathbb{Z}_{\geq 1}\mid
    \textstyle\sum_{s=1}^l\textstyle\binom{m_s+N-1}{m_s}\geq H-H^*
  \right\}\\
  K:=&\max\left\{k=0,\ldots,L\mid
    M\textstyle\sum_{s=1}^k m_s\textstyle\binom{m_s+N-1}{m_s}\leq N(H-H^*)
  \right\},
\end{align*}
we obtain the following upper bound:
\begin{equation}\label{eq_sing01_02}
  \lambda_{P_1}\leq
  \textstyle\frac{(M+N)H^*}{2} +
  \begin{cases}
    \frac{N(H-H^*)}{2m_1} & \text{if } K=0,\\
    \frac{N(H-H^*)+M\sum_{s=1}^{K}\binom{m_s+N-1}{m_s}\times(m_{K+1}-m_s)}{2m_{K+1}} & \text{if } 1\leq K \leq L-1,\\
    \frac{M(H-H^*)}{2} & \text{if } K= L.
  \end{cases}
\end{equation}

\begin{rem}
  When $N=1$ and $m_s=m+Q(s-1)$ $(s=1,2,\ldots)$, the local learning coefficient has already been obtained in \cite{Aoyagi2009}. 
  The value on the right-hand side of \eqref{eq_sing01_02} agrees with this exact value.

  When $N\geq2$ and $m_s=1+Q(s-1)$ $(s=1,2,\ldots)$,
  \eqref{eq_sing01_02} is consistent with the upper bounds obtained in
  \cite{Aoyagi2013}.  More precisely, \cite[bound~2]{Aoyagi2013} gives the
  second case of \eqref{eq_sing01_02}, whereas \cite[bound~3]{Aoyagi2013}
  gives the third case.

  Sufficient conditions for the linear independence condition
  \eqref{eq_sing01_01} are given in
  Proposition~\ref{prop_lin_independent01}.
  The conditions in Proposition~\ref{prop_lin_independent01} give a broad and
  natural sufficient condition in the continuous-input setting considered here.
  Discrete or degenerate input distributions and symmetric weight configurations
  require separate consideration.
\end{rem}

\subsubsection{Upper bound for the learning coefficient when $N=1$}
Here we compare our upper bound only in the case \(N=1\), for which the learning coefficient
has been obtained in previous work \cite{Aoyagi2009} when \(\sigma=\tanh\).
Assume that, for every positive integer \(S\), the following family of random variables is linearly independent:
\begin{equation}\label{eq_sing02_01}
  \begin{array}{ll}
    \sigma(b_i^{*}X)
      & (1\leq i \leq H^*),\\
    \sigma^{\prime}(b_i^{*}X)X
      & (1\leq i \leq H^*),\\
    \displaystyle\sigma^{(s)}(b_{H^*}^{*}X)X^s
      & (2\leq s\leq S).
  \end{array}
\end{equation}
Consider the realization parameter, which is defined only when the true number
of hidden-layer units satisfies $H^*\geq1$,
\begin{align*}
  P_2:&
  \begin{cases}
    a_{\cdot,i}=a_{\cdot,i}^*~(1\le i\le H^*),\quad
    a_{\cdot,i}=0~(H^*+1\le i\le H),\\
    b_i=b_i^*~(1\le i\le H^*),\quad
    b_i=b_{H^*}^*~(H^*+1\le i\le H).
  \end{cases}
\end{align*}
Then the Main Theorem applies with
\[
  (r,\alpha,\beta,\gamma,m_s,n_s)
  =
  \left((M+1)H^*,\, M(H-H^*),\, H-H^*,\,\infty,\, s,\,M-\mathbf{1}_{s=1}\right).
\]
When $M=1$, the term with $s=1$ does not appear. Strictly speaking, therefore,
we apply the Main Theorem after reindexing the sequence as $m_s=s+1$ and $n_s=1$.
Using
\begin{align*}
  L:=&
    \min\left\{
    l\in\mathbb{Z}_{\geq 1}\mid 
    Ml\geq M(H-H^*)+1
    \right\}\\
  K:=&
    \max\left\{k=0,\ldots,L\mid
    M\textstyle\binom{k+1}{2}\leq H-H^*+1
    \right\},
\end{align*}
we obtain
\begin{equation}\label{eq_sing02_02}
  \lambda_{P_2}\leq 
  \frac{(M+1)H^*}{2} + 
  \begin{cases}
    \frac{H-H^*}{2} & \text{if } K=0,\\
    \frac{H-H^*-K + M\binom{K+1}{2}}{2(K+1)} & \text{if } 1\leq K \leq L-1,\\
    \frac{M(H-H^*)}{2} & \text{if } K= L.
  \end{cases}
\end{equation}
Thus, when $H^*\geq1$, we obtain the following upper bound for the learning coefficient
$\lambda$ of the three-layer neural network:
\begin{equation}\label{eq_comparison01}
  \lambda\leq \min\{\lambda_{P_1}, \lambda_{P_2}\}
  \leq \min\{\lambda^{(1)}, \lambda^{(2)}\}.
\end{equation}
Here, $\lambda^{(1)}$ and $\lambda^{(2)}$ denote the right-hand sides of
\eqref{eq_sing01_02} and \eqref{eq_sing02_02}, respectively.

In what follows, we consider the three types of activation functions listed in
Table~\ref{table_lambdaNN01}.
\begin{table}[htbp]
  \centering
  \caption{Non-polynomial activation functions considered}
  \label{table_lambdaNN01}
  \begin{tabular}{cc}
    Activation function $\sigma(x)$ & $\{s\mid \sigma^{(s)}(0)\neq 0\}$ \\ \hline
    $e^x-1,~xe^x,~ x\tanh(\log(1+e^x))$ & $(1,2,3,4,5,\ldots)$ \\\hline
    $x/(1+e^{-x}),~ x\Phi(x)$ & $(1,2,4,6,8,\ldots)$ \\\hline
    $\tanh(x), \sin(x), \arctan(x)$ & $(1,3,5,7,9,\ldots)$
  \end{tabular}
  \caption*{\footnotesize $\Phi(x)$: the distribution function of the one-dimensional standard normal distribution}
\end{table}

When $\sigma(x)=\exp(x)-1$, Remark~\ref{rem_mainthm01}(2) gives
$\lambda^{(2)}\leq\lambda^{(1)}$.
For the other activation functions, a straightforward calculation shows that
the necessary and sufficient condition for $\lambda^{(2)}\leq\lambda^{(1)}$ is
\[
\begin{aligned}
  H-H^*&\leq \max\{17,3M+8\}
    &&\text{for the swish type,}\\
  H-H^*&\leq
  \left\{
  \begin{array}{@{}ll@{}}
    9     & (M=1),\\[-0.2ex]
    4     & (M=2),\\[-0.2ex]
    M+3   & (M\geq3)
  \end{array}
  \right.
    &&\text{for the $\tanh$ type;}
\end{aligned}
\]
see Figure~\ref{fig_learning_coefficientsNN01}.
When $\sigma=\tanh$, the right-hand side of \eqref{eq_comparison01} agrees with
the learning coefficient obtained in \cite{Aoyagi2009}.

\begin{figure}[htbp]
  \centering
  \includegraphics[width=0.68\textwidth]{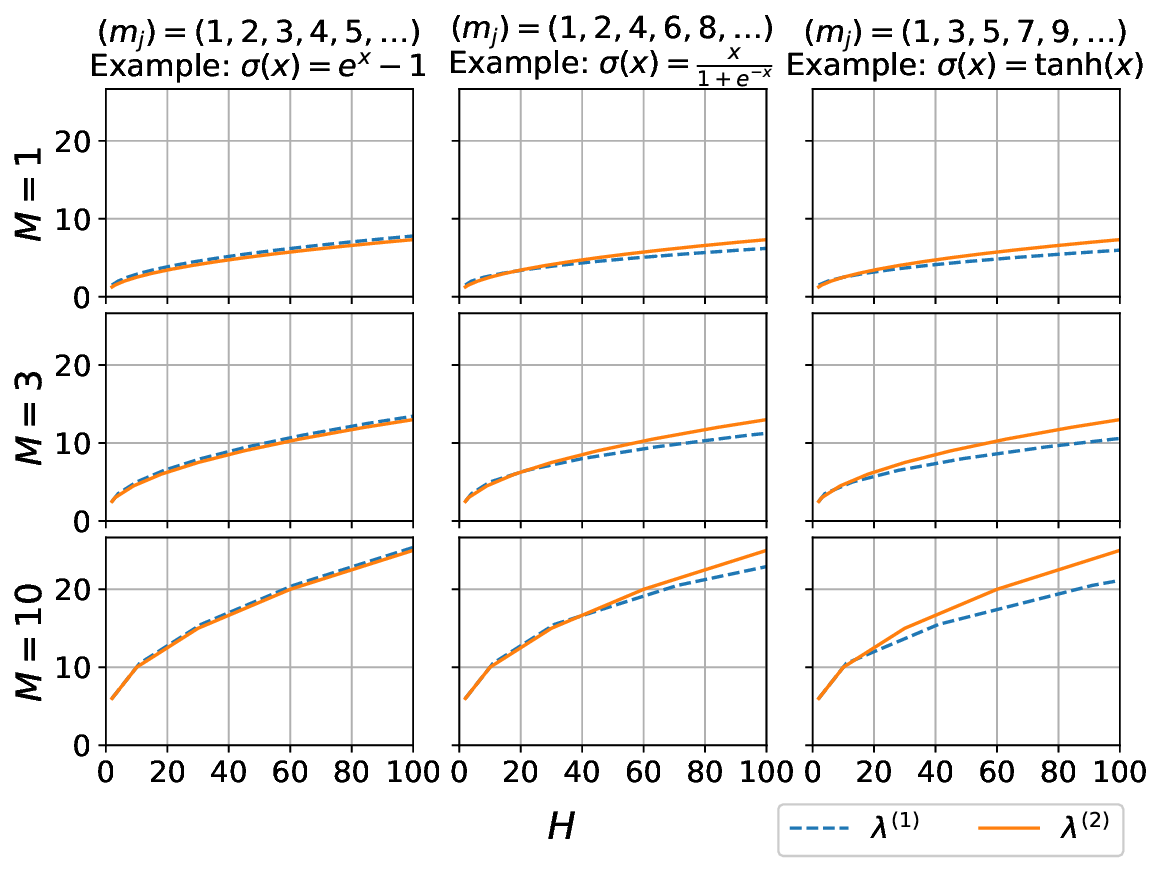}
  \caption{Upper bounds of the local learning coefficients at $P_1$ and $P_2$ for $(H^*,N)=(2,1)$}
  \label{fig_learning_coefficientsNN01}
\end{figure}

\subsection{The case where $\sigma$ is a polynomial function}\label{sec_ex_polynomial}

\subsubsection{The case $\sigma(x)=x$ (the reduced-rank regression model)}

  Let $A^*$ and $B^*$ be the true parameters, and let
  $R:=\textrm{rank}(A^*B^*)$.  In applying the Main Theorem below,
  we consider the non-boundary case $R<\min\{M,N,H\}$.
  Assume that the random variables
  $
    X_j ~ (1\leq j\leq N)
  $
  are linearly independent.
  After a suitable coordinate transformation around the realization parameter $P_1$, 
  the assumptions of the Main Theorem are satisfied with
  $
    (r,\alpha,\beta,\gamma,m_s,n_s)
    =(R(M+N-R),\ (M-R)(H-R),\ (N-R)(H-R),\ 1,\ 1,\ (M-R)(N-R)),
  $
  and we obtain the following upper bound:
  \begin{align*}
    \lambda_{P_1}
    \leq
    \begin{cases}
      \frac{HN-HR+MR}{2} & \text{if } N<H<M \text{ or } H\leq N<M,\\
      \frac{HM-HR+NR}{2} & \text{if } N\geq M \text{ and } N\geq H,\\
      \frac{MN}{2} & \text{if } H\geq M \text{ and } N< H.
    \end{cases}
  \end{align*}
  The learning coefficient of the reduced-rank regression model was obtained in \cite{Aoyagi2005}, 
  which classifies the specific value into four cases.
  The value of the upper bound obtained here agrees with the results in \cite[Cases~2--4]{Aoyagi2005}.
  On the other hand, a result corresponding to \cite[case~1]{Aoyagi2005} is not obtained from the Main Theorem.
  This gives an example in which the inequality in the Main Theorem is strict.
  For example, when $N=M=H=d\geq2$ and $R=0$, the upper bound obtained in this paper is $d^2/2$, whereas the learning coefficient is $3d^2/8$ if $d$ is even and $(3d^2+1)/8$ if $d$ is odd.

      \begin{rem}
        $r+\alpha+\beta$ does not coincide with the number of parameters. This suggests that there are parameters that do not contribute to the result of the Main Theorem, that is, redundant parameters.
      \end{rem}

\subsubsection{The case of a general polynomial $\sigma$}
  Here we consider only the case $H^*=0$. 
  Suppose that
  $$
    \sigma(x):=\textstyle\sum_{s=1}^S c_sx^{m_s},
    ~ c_s\neq 0,~ m_1<\cdots<m_S.
  $$
  Assume that the random variables
  $
    \textstyle\prod_{j=1}^N X_j^{h_j}
    ~
    (1\leq s\leq S,\ h_1+\cdots+h_N=m_s,\ h_j\geq 0)
  $
  are linearly independent.
  Then the assumptions of the Main Theorem are satisfied at $P_1$ with
  $
    (r,\alpha,\beta,\gamma,m_s,n_s)
    =(0,HM,HN,S,m_s,M\binom{m_s+N-1}{m_s}),
  $
  where
  \begin{align*}
    L=&
    \begin{cases}
      S
      & \text{if } \sum_{s=1}^S\binom{m_s+N-1}{m_s}<H,\\
      \min\left\{\,l=1,\ldots,S \,\middle|\, \sum_{s=1}^l \binom{m_s+N-1}{m_s}\geq H\,\right\}
      & \text{otherwise},
    \end{cases}\\
    n_s^\ast=&
    \begin{cases}
      M\binom{m_s+N-1}{m_s} & \text{if } s<L,\\
      M\min\left\{
        \binom{m_L+N-1}{m_L},\,
        H-\sum_{s=1}^{L-1}\binom{m_s+N-1}{m_s}
      \right\}
      & \text{if } s=L,
    \end{cases}\\
    K:=&
    \max\left\{
      k=0,\ldots,L
      \,\middle|\,
      \textstyle\sum_{s=1}^k m_sn_s^\ast\leq HN
    \right\},
  \end{align*}
  and we obtain the following upper bound:
  \begin{equation*}
    \lambda_{P_1}
    \leq
    \begin{cases}
      \frac{HN}{2m_1}
      & \text{if } K=0,\\
      \frac{HN+\sum_{s=1}^{K}(m_{K+1}-m_s)n_s^\ast}{2m_{K+1}}
      & \text{if } 1\leq K\leq L-1,\\
      \frac{\sum_{s=1}^L n_s^\ast}{2}
      & \text{if } K=L.
    \end{cases}
  \end{equation*}

  \begin{rem}
    For $H^*>0$, the Main Theorem can also be applied in some cases, but the resulting analysis
    takes a form different from that in the non-polynomial analytic case and requires a separate treatment.
    We therefore leave a systematic analysis of this case to future work.
  \end{rem}

\section{Proof of the Main Theorem}\label{sec_pf_mainthm}
  We present an outline of the proof here; see \ref{sec_DetailPf_mainthm} for details.
  We prove the claim by carrying out the four steps shown in Figure~\ref{fig3} and obtaining a normal crossing form.
  
  Step~1 consists of a total of $m_1$ blow-ups.
  In Step~2, we perform the coordinate transformation $a\mapsto a^\prime$ in a local chart selected using Condition~(ii).
  Since this local chart is only part of the full coordinate neighborhood,
the resulting value is not $\lambda_P$ itself, but rather an upper bound for it.
  Step~3 is the stage in which we perform blow-ups using $a^\prime$.
  In each Step~3-$k$ $(1\leq k \leq L-1)$, we perform $m_{k+1}-m_k$ blow-ups.
  Step~4 consists of one blow-up with respect to $(\theta,a)$, after which a normal crossing form is obtained.
  
  The upper bounds for $\lambda_P$ obtained at each step are as follows.
  In Step~1, we obtain $r/2+\beta/(2m_1)$, and in Step~3-$k$, we obtain
  \[
    \frac{r}{2}
    +
    \min\left\{
      \frac{\beta+\sum_{s=1}^{k-1}(m_k-m_s)n_s^\ast}{2m_k},
      \frac{\beta+\sum_{s=1}^k(m_{k+1}-m_s)n_s^\ast}{2m_{k+1}}
    \right\},
    \quad 1\leq k\leq L-1.
  \]
  In Step~4, we obtain
  \[
    \frac{r}{2}
    +
    \min\left\{
      \frac{\beta+\sum_{s=1}^{L-1}(m_L-m_s)n_s^\ast}{2m_L},
      \frac{\sum_{s=1}^L n_s^\ast}{2}
    \right\}.
  \]
  Therefore,
  \begin{align}
    \lambda_P & \leq 
    \frac{r}{2} + 
    \min
    \left\{
      \frac{\beta}{2m_1},
      \frac{\beta+\sum_{s=1}^{k}(m_{k+1}-m_s)n_s^\ast}{2m_{k+1}}
      \quad (1\leq k\leq L-1),
      \frac{\sum_{s=1}^L n_s^\ast}{2}
    \right\}.
    \label{eq_pf_mainthm00}
  \end{align}
  To determine the minimum in \eqref{eq_pf_mainthm00}, note that
  \[
    \frac{\beta+\sum_{s=1}^{k}(m_{k+1}-m_s)n_s^\ast}{2m_{k+1}}
    =
    \frac{\sum_{s=1}^{k}n_s^\ast}{2}
    +
    \frac{\beta-\sum_{s=1}^{k}m_sn_s^\ast}{2m_{k+1}}.
  \]
  Hence, for $k=0,\ldots,L-2$,
  \begin{align*}
    &
    \frac{\sum_{s=1}^{k}n_s^\ast}{2}
    +
    \frac{\beta-\sum_{s=1}^{k}m_sn_s^\ast}{2m_{k+1}}
    <
    \frac{\sum_{s=1}^{k+1}n_s^\ast}{2}
    +
    \frac{\beta-\sum_{s=1}^{k+1}m_sn_s^\ast}{2m_{k+2}}
    \\
    &\qquad\Longleftrightarrow\qquad
    \beta<\sum_{s=1}^{k+1}m_sn_s^\ast.
  \end{align*}
  Also,
  \[
    \frac{\sum_{s=1}^{L-1}n_s^\ast}{2}
    +
    \frac{\beta-\sum_{s=1}^{L-1}m_sn_s^\ast}{2m_L}
    <
    \frac{\sum_{s=1}^{L}n_s^\ast}{2}
    \Longleftrightarrow
    \beta<\sum_{s=1}^{L}m_sn_s^\ast.
  \]
  Therefore \eqref{eq_pf_mainthm00} gives \eqref{eq_mainthm01}.
  The proof is complete.

\begin{figure}
  \centering
  \begin{tikzpicture}[
    x=0.4cm, y=1cm,
    >=Stealth,
    nonterm/.style    ={circle,draw=black,line width=0.85pt,minimum size=3.9mm,inner sep=0pt},
    nontermBig/.style ={circle,draw=black,line width=1.05pt,minimum size=6.8mm,inner sep=0pt},
    term/.style       ={circle,draw=blue, line width=0.85pt,minimum size=3.9mm,inner sep=0pt},
    termBig/.style    ={circle,draw=blue, line width=1.05pt,minimum size=6.8mm,inner sep=0pt},
    halo/.style       ={circle,draw=red,  dashed,line width=1.0pt,minimum size=7.6mm,inner sep=0pt},
    ct/.style         ={font=\scriptsize} 
  ]

  \node[nontermBig] (S0)  at (0,0)   {};
  \node[termBig]    (T1)  at (0,-2)  {}; 

  \node[nontermBig] (S2)  at (3,0)   {}; 
  \node[termBig]    (S2d) at (3,-2)  {}; 

  \node[nontermBig] (S4)  at (9,0)   {}; 

  \node[nonterm]    (A0)  at (11.5,0) {};
  \node[halo]        at (11.5,0) {};

  \node[term]       (A0d) at (11.5,-2) {};
  \node[nonterm]    (B0)  at (14,0)    {};
  \node[term]       (B0d) at (14,-2)   {};
  \node[nonterm]    (C0)  at (17,0)    {}; 

  \node[nonterm]    (D0)  at (20,0)    {}; 

  \node[term]       (D0d) at (20,-2)   {};
  \node[nonterm]    (E0)  at (22.5,0)  {};
  \node[term]       (E0d) at (22.5,-2) {};
  \node[nonterm]    (F0)  at (25.5,0)  {}; 

  \node[term]       (F0d) at (25.5,-2) {};
  \node[term]       (G0)  at (28,0)    {}; 

  \def\ctSeg{2.5} 
  \path (S2) ++(\ctSeg,0)  coordinate (S2midA);
  \path (S4) ++(-\ctSeg,0) coordinate (S2midB);

  \path (B0) ++(1.2,0) coordinate (B0dot);
  \path (E0) ++(1.2,0) coordinate (E0dot);

  \draw[->,line width=1.0pt] (S0) -- node[left,ct]  {CT1} (T1);
  \draw[->,line width=1.0pt] (S0) -- node[above,ct] {CT2} (S2);

  \draw[->,line width=1.0pt] (S2) -- node[left,ct]  {CT3} (S2d);

  \draw[->,line width=1.0pt] (S2) -- node[above,ct] {CT4} (S2midA);
  \draw[dotted,line width=1.0pt] (S2midA) -- (S2midB);
  \draw[->,line width=1.0pt] (S2midB) -- node[above,ct] {CT4} (S4);

  \draw[->,line width=1.0pt] (S4) -- (A0);

  \draw[->,line width=1.0pt] (A0) -- (A0d); 
  \draw[->,line width=1.0pt] (A0) -- node[above,ct] {CT7(1)}      (B0);
  \draw[->,line width=1.0pt] (B0) -- node[left,ct]  {CT5--CT6(1)} (B0d);
  \draw[dotted,line width=1.0pt] (B0) -- (B0dot);
  \draw[->,line width=1.0pt] (B0dot) -- node[above,ct] {CT7(1)}   (C0);

  \draw[dotted,line width=1.0pt] (C0) -- (D0);

  \draw[->,line width=1.0pt] (D0) -- (D0d); 
  \draw[->,line width=1.0pt] (D0) -- node[above,ct] {CT7(L-1)}      (E0);
  \draw[->,line width=1.0pt] (E0) -- node[left,ct]  {CT5--CT6(L-1)} (E0d);
  \draw[dotted,line width=1.0pt] (E0) -- (E0dot);
  \draw[->,line width=1.0pt] (E0dot) -- node[above,ct] {CT7(L-1)}   (F0);

  \draw[->,line width=1.0pt] (F0) -- node[left,ct]  {CT8} (F0d);
  \draw[->,line width=1.0pt] (F0) -- node[above,ct] {CT9} (G0);

  \def\yStep{1.25}
  \def\yVBot{0.70}
  \def\yVTop{1.45}

  \foreach \N in {S0,S4,A0,C0,D0,F0,G0}{
    \draw[dashed,line width=0.9pt] ($(\N)+(0,\yVBot)$) -- ($(\N)+(0,\yVTop)$);
  }

  \draw[dotted,<->,line width=1.0pt] ($(S0)+(0,\yStep)$) -- node[above]{Step1}       ($(S4)+(0,\yStep)$);
  \draw[dotted,<->,line width=1.0pt] ($(S4)+(0,\yStep)$) -- node[above]{Step2}       ($(A0)+(0,\yStep)$);
  \draw[dotted,<->,line width=1.0pt] ($(A0)+(0,\yStep)$) -- node[above]{Step3-1}     ($(C0)+(0,\yStep)$);
  \draw[dotted,<->,line width=1.0pt] ($(C0)+(0,\yStep)$) -- node[above]{\(\cdots\)}  ($(D0)+(0,\yStep)$);
  \draw[dotted,<->,line width=1.0pt] ($(D0)+(0,\yStep)$) -- node[above]{Step3-(L-1)} ($(F0)+(0,\yStep)$);
  \draw[dotted,<->,line width=1.0pt] ($(F0)+(0,\yStep)$) -- node[above]{Step4}       ($(G0)+(0,\yStep)$);

  \end{tikzpicture}
  \caption{Procedure of the coordinate transformations $\pi$ in the proof of the Main Theorem}
  \caption*
  {\footnotesize The red dashed ring at the endpoint of Step~2 indicates that the subsequent operations are performed on a restricted neighborhood (a subset) around that point.}
  \label{fig3}
\end{figure}

\section{Conclusion}\label{sec_conclusion}
In this paper, we derived a broadly applicable formula that gives an upper bound for local learning coefficients
of three-layer neural networks.
By applying this formula, we obtained upper bounds for the learning coefficients of
three-layer neural networks in general settings for non-polynomial real-analytic activation functions,
and, in the polynomial case, under the restriction that the true distribution has no hidden units.

As future work, it remains important to analyze the loci where the rank condition in Condition~(ii) fails.
Such an analysis may lead to sharper upper bounds, or to additional conditions under which equality holds in the upper-bound formula derived in this paper.  
Another direction is to extend the methodology developed in this study to deep neural networks beyond the three-layer setting.
It also remains important to treat polynomial activation functions when the true distribution has hidden units, and to investigate empirically whether these upper bounds provide useful information in model comparison and model-selection problems.

\section*{Acknowledgments}
I am grateful to Professor Joe Suzuki of the University of Osaka for teaching me the basics of Bayesian theory and for providing a research topic that bridges algebraic geometry and statistics. 
I also thank him for his valuable comments and advice and for carefully reading this manuscript.
I express my gratitude to Professor Sumio Watanabe, who proposed the concept of the learning coefficient.

\appendix  

\def\thesection{Appendix \Alph{section}}
\renewcommand{\thesubsection}{\Alph{section}.\arabic{subsection}}

\renewcommand{\theprop}{\Alph{section}.\arabic{prop}}
\renewcommand{\thelem}{\Alph{section}.\arabic{lem}}
\renewcommand{\thecor}{\Alph{section}.\arabic{cor}}
\renewcommand{\thedfn}{\Alph{section}.\arabic{dfn}}
\renewcommand{\theex}{\Alph{section}.\arabic{ex}}
\renewcommand{\therem}{\Alph{section}.\arabic{rem}}
\renewcommand{\thethm}{\Alph{section}.\arabic{thm}}
\renewcommand{\theequation}{\Alph{section}.\arabic{equation}}

\section{Verification of the Assumptions of the Main Theorem}\label{sec_app_mainthm}
    In this section, we verify that the three-layer neural network considered in Section~\ref{exsec01} satisfies the assumptions of the Main Theorem.
    Let $\phi$ and $\Phi$ be $\mathbb{R}^M$-valued analytic functions representing the statistical model with parameters $(\theta,a,b)$ and the true distribution, respectively:
    \begin{equation}\label{eq_exNNcond02}
      Y = \phi(X|\theta,a,b) + \mathcal{N},\quad Y = \Phi(X) + \mathcal{N}.
    \end{equation}
    We assume that the origin $(\theta,a,b)=(0,0,0)$ is a realization parameter.
    Letting $Z:=Y-\Phi(X)$, the log-likelihood ratio function can be written as
    \begin{equation}\label{eq_exNNcond00}
      f
      = -Z^\top\left(\phi(X|\theta,a,b) - \Phi(X)\right)
      + \frac{1}{2}\left\|\phi(X|\theta,a,b) - \Phi(X)\right\|^2.
    \end{equation}
    We denote the $k$-th components of $Z$, $\phi(X|\theta,a,b)$, and $\Phi(X)$ by $Z_k$, $\phi_k$, and $\Phi_k$, respectively $(1\leq k\leq M)$.

 \subsection{On $P_1$}\label{subsec_NN01}
      We first consider the case where $\sigma$ is not a polynomial, as treated in Section~\ref{sec_ex_nonpolynomial}.
      By translating the parameters so that the realization parameter is at the origin, we can write
      \[
        \phi_k=
        \sum_{i=1}^{H^*}(a_{k,i}+a_{k,i}^*)
        \sigma
        \left(
          (b_i+b_i^*)^\top X
        \right)
        +\sum_{i=H^*+1}^H a_{k,i}\sigma
        \left(b_i^\top X\right).
      \]
      We divide the parameters into the following three groups:
      \begin{gather*}
        \boldsymbol{\theta}:=\{a_{\cdot,i}, b_i\mid 1\leq i\leq H^*\},\\
        \boldsymbol{a}:=\{a_{\cdot,i}\in\mathbb{R}^M\mid H^*+1\leq i\leq H\},\quad
        \boldsymbol{b}:=\{b_i\in\mathbb{R}^N\mid H^*+1\leq i\leq H\},
      \end{gather*}
      and verify that the assumptions of the Main Theorem are satisfied under this parametrization.
      We have
      \begin{align*}
        \left.
          \frac{\partial\phi_k}{\partial a_{k,i}}
        \right|_{(\boldsymbol{\theta},\boldsymbol{a},\boldsymbol{b})=0}
        =&\sigma\left(b_i^{*\top}X\right)
        & (1\leq i\leq H^*, 1\leq k\leq M), \\
        \left.
          \frac{\partial\phi_k}{\partial b_{i,j}}
        \right|_{(\boldsymbol{\theta},\boldsymbol{a},\boldsymbol{b})=0}
        =&a_{k,i}^*\sigma^\prime\left(b_i^{*\top}X\right)X_j
        & (1\leq i\leq H^*, 1\leq j\leq N),
      \end{align*}
      and
      \begin{align*}
        \phi_k(X|\boldsymbol{\theta}=0)-\Phi_k(X)
        =&\sum_{i=H^*+1}^H a_{k,i}\sigma\left(b_i^\top X\right)\\
        =&\sum_{i=H^*+1}^H a_{k,i}\sum_{s=1}^\infty
        \frac{\sigma^{(m_s)}(0)}{m_s!}\left(\textstyle\sum_{j=1}^N b_{i,j}X_j\right)^{m_s}\\
        =&\sum_{s=1}^\infty
        \sum_{|\boldsymbol{h}|=m_s}
        \sum_{i=H^*+1}^H
        a_{k,i}\prod_{j=1}^N b_{i,j}^{h_j}
        \cdot
        \frac{\sigma^{(m_s)}(0)\prod_{j=1}^NX_j^{h_j}}{\prod_{j=1}^Nh_j!}\\
        =&\sum_{s=1}^\infty\sum_{|\boldsymbol{h}|=m_s}
        g_{s,\boldsymbol{h},k}(\boldsymbol{a},\boldsymbol{b})R_{s,\boldsymbol{h}}(X).
      \end{align*}
      Here, for $h_j\geq 0$, we define the multi-index $\boldsymbol{h}:=(h_1,\ldots,h_N)$ and $|\boldsymbol{h}|:=h_1+\cdots+h_N$, and define the polynomial
      $g_{s,\boldsymbol{h},k}(\boldsymbol{a},\boldsymbol{b}):=\sum_{i=H^*+1}^Ha_{k,i}\prod_{j=1}^N b_{i,j}^{h_j}$
      and the random variable $R_{s,\boldsymbol{h}}(X):=\sigma^{(|\boldsymbol{h}|)}(0)\prod_{j=1}^NX_j^{h_j}/\prod_{j=1}^Nh_j!$.
      Noting that, for each $s$, the number of such polynomials $g$ is $n_s^\prime:=\binom{m_s+N-1}{m_s}$, 
      we impose an appropriate order on the index $\boldsymbol{h}$ and relabel it by a new index $n$, so that
      \begin{equation}\label{eq_exNNcondP1_01}
        \phi_k(X|\boldsymbol{\theta}=0)-\Phi_k(X)
        =\sum_{s=1}^\infty\sum_{n=1}^{n_s^\prime}g_{s,n,k}(\boldsymbol{a},\boldsymbol{b})R_{s,n}(X).
      \end{equation}
      Substituting (\ref{eq_exNNcondP1_01}) into (\ref{eq_exNNcond00}), we obtain
      \begin{align*}
        &f(X|\theta=0,\boldsymbol{a},\boldsymbol{b})\\
        =& -\sum_{k=1}^M Z_k
        \sum_{s=1}^\infty \sum_{n=1}^{n_s^\prime} g_{s,n,k}(\boldsymbol{a},\boldsymbol{b})R_{s,n}(X)
        + \frac{1}{2}\sum_{k=1}^M
        \left\{
          \sum_{s=1}^\infty\sum_{n=1}^{n_s^\prime} g_{s,n,k}(\boldsymbol{a},\boldsymbol{b})R_{s,n}(X)
        \right\}^2\\
        =&  \sum_{s=1}^\infty \sum_{k=1}^M \sum_{n=1}^{n_s^\prime}
        g_{s,n,k}(\boldsymbol{a},\boldsymbol{b})
        (-Z_k R_{s,n}(X))
        + \sum_{s=1}^\infty h_s(\boldsymbol{a},\boldsymbol{b})W_s.
      \end{align*}
      Here, the \(W_s\) are random variables, and $h_s(\boldsymbol{a},\boldsymbol{b})\in$ ideal $(g_{1,1,1}(\boldsymbol{a},\boldsymbol{b}),g_{1,1,2}(\boldsymbol{a},\boldsymbol{b}),\ldots)^2$.
      By regarding $(k,n)$ as a single index $n$, the above expression has the same form as in the Main Theorem when
      $
        (\gamma,m_s,n_s)=(\infty,m_s,M\cdot\binom{m_s+N-1}{m_s}).
      $
      Moreover, conditions (i) and (iv) are satisfied.

      We next verify condition (ii) of the Main Theorem.
      For each fixed $k\in\{1,\ldots,M\}$ and for $i=1,\ldots,H-H^*$, we have
      \[
        \left.
          \textstyle\frac{\partial g_{s,\boldsymbol{h},k}}{\partial a_{k,H^*+i}}
        \right|_{\boldsymbol{a}=0}
        =
        \textstyle\prod_{j=1}^N b_{H^*+i,j}^{h_j}.
      \]
      Therefore, by Lemma~\ref{lem_genVander} stated below, the matrix obtained by arranging
      $H-H^*$ of these functions $g_{s,\boldsymbol{h},k}$ in increasing order of degree is nonsingular
      for generic $\boldsymbol{b}\neq 0$.
      The full Jacobian is the block diagonal matrix consisting of $M$ copies of this nonsingular matrix,
      and hence is nonsingular for the same generic $\boldsymbol{b}\neq 0$.

      Finally, we verify condition (iii) of the Main Theorem.
      For real numbers ${t_i, u_{s,n,k}}$, assume that
      \[
      \sum_{i=1}^r t_i
      \left.
        \frac{\partial f}{\partial \theta_i}
      \right|_{(\boldsymbol{\theta},\boldsymbol{a},\boldsymbol{b})=0}
      + \sum_{s=1}^L \sum_{k=1}^M \sum_{n=1}^{n_s^\prime}
      u_{s,n,k}Z_{k}R_{s,n}(X) = 0
      \quad q\text{-a.s.}
      \]
      Since (\ref{eq_exNNcond00}) implies
      \begin{equation*}
        \forall i=1,\ldots,r,\quad
          \left.\frac{\partial f}{\partial \theta_i}\right|_{(\boldsymbol{\theta},\boldsymbol{a},\boldsymbol{b})=0}
        = -\sum_{k=1}^M Z_k
            \left.\frac{\partial \phi_k}{\partial \theta_i}\right|_{(\boldsymbol{\theta},\boldsymbol{a},\boldsymbol{b})=0},
      \end{equation*}
      we can transform the above equality as follows:
      \[
        \sum_{k=1}^M
        \left\{
          \sum_{i=1}^r t_i
          \left.
            \frac{\partial \phi_k}{\partial \theta_i}
          \right|_{(\boldsymbol{\theta},\boldsymbol{a},\boldsymbol{b})=0}
          +\sum_{s=1}^L \sum_{n=1}^{n_s^\prime} u_{s,n,k}R_{s,n}(X)
        \right\}Z_k=0
        \quad q\text{-a.s.}
      \]
      Since the conditional variance of the left-hand side given $X$ is $0$, and since $Z|X\sim\mathcal{N}$, we obtain
      \[
        \forall k=1,\ldots,M\quad
        \sum_{i=1}^r t_i
        \left.
          \frac{\partial \phi_k}{\partial \theta_i}
        \right|_{(\boldsymbol{\theta},\boldsymbol{a},\boldsymbol{b})=0}
        +\sum_{s=1}^L \sum_{n=1}^{n_s^\prime}
        u_{s,n,k}R_{s,n}(X)
        =0
        \quad q_X\text{-a.s.}
      \]
      Since every column vector of $A^*$ is assumed to be nonzero,
      for each $i$ there exists $k$ such that $a_{k,i}^*\neq 0$.
      For this choice of $k$, the linear independence assumption
      \eqref{eq_sing01_01} implies that
      \[
        t_i=0,
        \qquad
        u_{s,n,k}=0
        \quad
        (1\leq s\leq L,~1\leq n\leq n_s^\prime).
      \]
      Repeating this argument for each $i$, we obtain $t_i=0$ for all $i$.
      Substituting this into the previous identity and using again the linear independence of ${R_{s,n}(X)}$, we conclude that $u_{s,n,k}=0$ for all $k,s,n$.

      Therefore, at the realization parameter $P_1$, we have verified that the assumptions of the Main Theorem are satisfied when
      $
        (r,\alpha,\beta,\gamma,m_s,n_s)=((M+N)H^*,M(H-H^*),N(H-H^*),\infty,m_s,M\binom{m_s+N-1}{m_s}).
      $

      Finally, we consider the case where $\sigma$ is a polynomial, treated in Section~\ref{sec_ex_polynomial}, under the restriction that $H^*=0$.
      Since $\sigma(x)=\sum_{s=1}^S c_sx^{m_s}$, the same proof as above applies, except that the range of the index $s$ changes from $1,\ldots,\infty$ to $1,\ldots,S$, and there is no parameter corresponding to $\theta$.
      Hence, at the realization parameter $P_1$, the assumptions of the Main Theorem are satisfied when
      $
        (r,\alpha,\beta,\gamma,m_s,n_s)=(0,MH,NH,S,m_s,M\binom{m_s+N-1}{m_s}).
      $
      
      \begin{lem}\label{lem_genVander}
        Let
        $
        \boldsymbol{b}_i:=
        \begin{pmatrix}
          b_{i,1} & b_{i,2} & \cdots & b_{i,N}
        \end{pmatrix}
        \in\mathbb{R}^N,~
        B:=(b_{i,j})_{1\leq i \leq H, 1\leq j\leq N}
        \in\mathbb{R}^{H\times N}.
        $
        For any choice of $H$ distinct monomials $m_1,\ldots,m_H$, define
        $M(B):=\left(m_k(\boldsymbol{b}_i)\right)_{1\leq i,k\leq H}$.
        Then $M(B)$ is nonsingular for generic $B$.
      \end{lem}
      
      \begin{proof}
        Since $\det M(B)$ is a polynomial in the entries of $B$, it suffices to show that
        $
          \{B\in\mathbb{R}^{H\times N}\mid\det M(B)\neq 0\}\neq \emptyset.
        $
        Write the chosen $H$ monomials as
        $m_k(\boldsymbol{x}):=\prod_{j=1}^N x_j^{h_{j,k}}$
        (where $(h_{1,k},\ldots,h_{N,k})$ are distinct),
        and let $D$ be the maximum total degree among them.
        Take any integer $L>D$.
        For any pairwise distinct positive real numbers $t_1,\ldots,t_H$, consider the case where
        $
          \boldsymbol{b}_i=\begin{pmatrix}
            t_i & t_i^L & t_i^{L^2} & \cdots & t_i^{L^{N-1}}
          \end{pmatrix}
          \quad(1\leq i \leq H).
        $
        Then
        \begin{align}
          m_k(\boldsymbol{b}_i)
          =\prod_{j=1}^N\left(t_i^{L^{j-1}}\right)^{h_{j,k}}
          =t_i^{\sum_{j=1}^N h_{j,k}L^{j-1}}
          =t_i^{E_k}
          ,\quad
          E_k:=\sum_{j=1}^N h_{j,k}L^{j-1}, \label{eq_genVander_01}
        \end{align}
        for $k=1,\ldots,H$.
        Since each $h_{j,k}<L$, the uniqueness of base-$L$ expansions implies that, if the $(h_{1,k},\ldots,h_{N,k})$ are pairwise distinct, 
        then the corresponding $E_k$ are also pairwise distinct.
        Therefore,
        $
          M(B)=\left(t_i^{E_k}\right)_{1\leq i,k\leq H},
          ~ E_i\neq E_j\ (i\neq j).
        $
        Thus $M(B)$ is a generalized Vandermonde matrix, and by its standard property \cite{RobbinSalamon2000}, we have $\det M(B)\neq 0$.
        This completes the proof.
      \end{proof}

    \subsection{On $P_2$}\label{subsec_NN02}
      In this subsection, we consider the case $N=1$.
      Translating the parameters so that the realization parameter is at the origin, we can write
      \begin{align*}
        \phi_k=&
        \sum_{i=1}^{H^*-1}(a_{k,i}+a_{k,i}^*)
        \sigma\left(
          (b_i+b_i^*) X
        \right)
        +(a_{k,H^*}+a_{k,H^*}^*)\sigma
        \left(
          (b_{H^*}+b_{H^*}^*) X
        \right)\\
        &+\sum_{i=H^*+1}^H a_{k,i}\sigma\left(
          (b_i+b_{H^*}^*) X
        \right).
      \end{align*}
      Since every column vector of $A^*$ is assumed to be nonzero,
      after relabeling the output coordinates if necessary, we may assume that $a_{1,H^*}^*\neq 0$.
      We perform the coordinate transformation
      \begin{align*}
        a_{\cdot,H^*}^\prime := a_{\cdot,H^*}+\cdots +a_{\cdot,H} ,\quad
        b_{H^*}^\prime := b_{H^*}
        +\frac{\sum_{i=H^*+1}^H a_{1,i}b_{i}}
        {a_{1,H^*}^*-\sum_{i=H^*+1}^H a_{1,i}},
      \end{align*}
      and, with a slight abuse of notation, denote the transformed coordinates
      $a_{\cdot,H^*}^\prime$ and $b_{H^*}^\prime$ again by
      $a_{\cdot,H^*}$ and $b_{H^*}$, respectively.
      We divide the parameters into the following three groups:
      \begin{gather*}
        \boldsymbol{\theta}:=\{a_{\cdot,i}, b_i\mid 1\leq i\leq H^*\},\\
        \boldsymbol{a}:=\{a_{\cdot,i}\in\mathbb{R}^M\mid H^*+1\leq i\leq H\},\quad
        \boldsymbol{b}:=\{b_i\in\mathbb{R}\mid H^*+1\leq i\leq H\},
      \end{gather*}
      and verify that the assumptions of the Main Theorem are satisfied for this parametrization.
      We have
      \begin{align*}
        \left.
          \frac{\partial\phi_k}{\partial a_{k,i}}
        \right|_{(\boldsymbol{\theta},\boldsymbol{a},\boldsymbol{b})=0}
        =&
        \sigma\left(b_i^{*}X\right) 
        \quad (1\leq i\leq H^*,\ 1\leq k\leq M),
        \\
        \left.
          \frac{\partial\phi_k}{\partial b_i }
        \right|_{(\boldsymbol{\theta},\boldsymbol{a},\boldsymbol{b})=0}
        =&
        a_{k,i}^*\sigma^\prime\left(b_i^{*}X\right)X
        \quad (1\leq i\leq H^*).
      \end{align*}
      For $s\geq 1$, define
      \begin{align*}
        g_{s,k}(\boldsymbol{a},\boldsymbol{b})
        :=&
        (-1)^s
        \frac{a^*_{k,H^*}-\sum_{i=H^*+1}^H a_{k,i}}
        {(a^*_{1,H^*}-\sum_{i=H^*+1}^H a_{1,i})^{s}}
        \left(\sum_{i=H^*+1}^H a_{1,i}b_i \right)^s
        +\sum_{i=H^*+1}^H a_{k,i}b_i^s,\\
        R_s(X)
        :=&
        \frac{\sigma^{(s)}(b_{H^*}^{*}X)X^s}{s!}.
      \end{align*}
      Then $g_{s,k}$ is homogeneous of degree $s$ in $\boldsymbol{b}$.
      Moreover, by Taylor expanding around $b_{H^*}^*X$, we obtain
      \begin{align*}
        \phi_k(X|\boldsymbol{\theta}=0)-\Phi_k(X)
        =&
        \textstyle\left(a_{k,H^*}^*-\sum_{i=H^*+1}^H a_{k,i}\right)
        \sigma
        \textstyle\left(
          \left(
            b_{H^*}^*
            -\frac{\sum_{i=H^*+1}^H a_{1,i}b_i}
            {a_{1,H^*}^*-\sum_{i=H^*+1}^H a_{1,i}}
          \right)X
        \right)\\
        &+\textstyle\sum_{i=H^*+1}^H a_{k,i}
        \sigma\left((b_i+b_{H^*}^*)X\right)
        -a_{k,H^*}^*\sigma\left(b_{H^*}^*X\right)\\
        =&
        \sum_{s=1}^{\infty}
        g_{s,k}(\boldsymbol{a},\boldsymbol{b})R_s(X).
      \end{align*}
  Since \(g_{1,1}=0\), in the block \(s=1\) it is enough to consider only the terms corresponding to \(k=2,\ldots,M\).
  For these terms, set
  \[
    n_s:=
    \begin{cases}
      M-1 & \text{if }s=1,\\
      M   & \text{if }s\geq 2.
    \end{cases}
  \]
  We then verify that the assumptions of the Main Theorem are satisfied for
  \[
    (r,\alpha,\beta,\gamma,m_s,n_s)
    =
    \left(
      (M+1)H^*,\,
      M(H-H^*),\,
      H-H^*,\,
      \infty,\,
      s,\,
      n_s
    \right).
  \]
  Here we verify only Condition (ii); the other conditions follow by the same argument as in Section~\ref{subsec_NN01}.

  First suppose that $M\geq2$.
  For $k=2,\ldots,M$, we have
  \[
    g_{1,k}
    =
    -\frac{
      a^*_{k,H^*}-\sum_{\ell=H^*+1}^{H}a_{k,\ell}
    }{
      a^*_{1,H^*}-\sum_{\ell=H^*+1}^{H}a_{1,\ell}
    }
    \sum_{i=H^*+1}^{H}a_{1,i}b_i
    +
    \sum_{i=H^*+1}^{H}a_{k,i}b_i .
  \]
  Hence
  \[
    \left.
    \frac{\partial g_{1,k}}{\partial a_{1,H^*+i}}
    \right|_{\boldsymbol a=0}
    =
    -\frac{a^*_{k,H^*}}{a^*_{1,H^*}}b_{H^*+i},
    \qquad
    \left.
    \frac{\partial g_{1,k}}{\partial a_{k,H^*+i}}
    \right|_{\boldsymbol a=0}
    =
    b_{H^*+i}.
  \]
  On the other hand, for $s\geq2$, the first term of $g_{s,k}$ is of order at least two
  in $\boldsymbol a$ and therefore does not contribute to the $a$-Jacobian at
  $\boldsymbol a=0$.
  Thus
  \[
    \left.
    \frac{\partial g_{s,k}}{\partial a_{\ell,H^*+i}}
    \right|_{\boldsymbol a=0}
    =
    \begin{cases}
      b_{H^*+i}^{\,s} & \text{if } \ell=k,\\
      0 & \text{if } \ell\neq k.
    \end{cases}
  \]

  We order the components of $g$ as
  \[
    g_{2,1},\ldots,g_{H-H^*+1,1},
    g_{1,2},\ldots,g_{H-H^*,2},
    \ldots,
    g_{1,M},\ldots,g_{H-H^*,M},
  \]
  and order the components of $\boldsymbol{a}$ as
  \[
    a_{1,H^*+1},\ldots,a_{1,H},
    a_{2,H^*+1},\ldots,a_{2,H},
    \ldots,
    a_{M,H^*+1},\ldots,a_{M,H}.
  \]
  With this ordering, the Jacobian matrix is block lower triangular, and therefore
  \begin{gather*}
    \det
    \left.
      \frac{\partial g}{\partial a}
    \right|_{\boldsymbol a=0}
    =
    (\det B_1)^{M-1}\det B_2,\\
    B_1:=
    \left(
      b_{H^*+i}^{\,j}
    \right)_{1\leq i,j\leq H-H^*},
    \qquad
    B_2:=
    \left(
      b_{H^*+i}^{\,1+j}
    \right)_{1\leq i,j\leq H-H^*}.
  \end{gather*}
  If $b_{H^*+1},\ldots,b_H$ are chosen to be distinct nonzero real numbers,
  then both $B_1$ and $B_2$ are nonsingular, because each is obtained from an ordinary
  Vandermonde matrix by multiplying rows by nonzero factors.
  Hence
  $
    \det
    \left.
      \frac{\partial g}{\partial a}
    \right|_{\boldsymbol a=0}
    \neq 0.
  $
  This verifies Condition (ii).

  When $M=1$, the same argument shows that the determinant of the corresponding
  Jacobian is $\det B_2$.
  Hence the Jacobian is nonsingular for the same choice of
  $b_{H^*+1},\ldots,b_H$, and Condition (ii) also holds in this case.

    \subsection{Assumption on linear independence}\label{subsec_lin_independent}
      We show that, when $\sigma$ is a non-polynomial real analytic function, the linear independence conditions (\ref{eq_sing01_01}) and (\ref{eq_sing02_01}) hold under the following conditions.
      Throughout this subsection, we use the same notation as in Section~\ref{exsec01}.
      \begin{prop}\label{prop_lin_independent01}
        Assume that the following three conditions hold.
        \begin{itemize}
          \item[(1)]
            The distribution $q_X$ of $X$ is absolutely continuous with respect to Lebesgue measure on an open neighborhood $U\subset\mathbb{R}^N$ of the origin, and its density is positive a.e.\ on $U$.
          \item[(2)]
            $b_1^*,\ldots,b_{H^*}^*\in\mathbb{R}^N$ satisfy $b_i^*\neq 0$ and $b_{i}^*\neq \pm b_{j}^*$ for $i\neq j$.
          \item[(3)]
            $\sigma$ is a non-polynomial real analytic function satisfying $\sigma(0)=0$.
        \end{itemize}
        Then the family in (\ref{eq_sing01_01}) is linearly independent.
        Furthermore, when $N=1$, the family in (\ref{eq_sing02_01}) is also linearly independent.
      \end{prop}
      
      \begin{proof}
        We prove only (\ref{eq_sing01_01}); the proof for (\ref{eq_sing02_01}) is analogous.
        Let $a_i\in\mathbb{R}$, $c_{i,j}\in\mathbb{R}$, and $d_{\boldsymbol{h}}\in\mathbb{R}$, and set
        $
          c_i:=(c_{i,1},\ldots,c_{i,N})^\top\in\mathbb{R}^N
          ~ (1\leq i\leq H^*).
        $
        Assume that
        \begin{equation}\label{eq_prop_lin_independent01_2}
          \sum_{i=1}^{H^*} a_i\sigma(b_i^{*\top} X)
          +\sum_{i=1}^{H^*}\sum_{j=1}^N c_{i,j}\sigma^\prime(b_i^{*\top} X)X_j
          + \sum_{1\leq |\boldsymbol{h}|\leq S}d_{\boldsymbol{h}}X^{\boldsymbol{h}}
          =0\quad
          q_X\text{-a.s.}
        \end{equation}
        Let $F$ be a real analytic function on $U$.
        If $F(X)=0$ holds $q_X\text{-a.s.}$, then (1) implies that $F(x)=0$ for Lebesgue-a.e.\ $x\in U$.
        Since $F$ is real analytic on $U$, it follows that $F$ vanishes identically on $U$.     
        Therefore, (\ref{eq_prop_lin_independent01_2}) holds for all $x\in U$.
        In particular, for any $u\in\mathbb{R}^N$, if $\epsilon>0$ is sufficiently small and $x=tu$ with $|t|<\epsilon$, then
        \begin{equation*}
          \sum_{i=1}^{H^*} a_i\sigma((b_i^{*\top} u)t)
          + t\sum_{i=1}^{H^*} (c_{i}^\top u)\sigma^\prime((b_i^{*\top} u)t)
          + \sum_{s=1}^S p_s(u)t^s
          =0
          ~;\quad
          p_s(u):=\sum_{|\boldsymbol{h}|=s} d_{\boldsymbol{h}}u^{\boldsymbol{h}}.
        \end{equation*}
        By (2), if we avoid the finitely many hyperplanes
        \[
          b_i^{*\top} u=0,~
          (b_i^*-b_j^*)^\top u=0,~
          (b_i^*+b_j^*)^\top u=0
          \quad (1\leq i,j\leq H^*),
        \]
        then we can choose $u$ so that
        \begin{equation}\label{eq_prop_lin_independent01_3}
          \alpha_i(u):=b_i^{*\top} u\neq 0,
          \quad
          |\alpha_i(u)|\neq |\alpha_j(u)|\quad(i\neq j)
        \end{equation}
        hold simultaneously. The set $V$ of such $u$ is a nonempty open subset of $\mathbb{R}^N$.
        Fix any $u\in V$.
        Letting $\beta_{i}(u):=c_i^\top u$, we obtain
        \begin{equation}\label{eq_prop_lin_independent01_1}
          \sum_{i=1}^{H^*} a_i\sigma(\alpha_i(u)t)
          + t\sum_{i=1}^{H^*} \beta_i(u) \sigma^\prime(\alpha_i(u)t)
          + \sum_{s=1}^S p_s(u)t^s
          =0
          \quad
          (|t|<\epsilon).
        \end{equation}        
        Next, consider the Taylor expansion of $\sigma$ around the origin,
        $
          \sigma(t)=\sum_{m=1}^\infty s_m t^{m}.
        $
        Since $\sigma$ is not a polynomial by (3), the set $\{m\geq 1\mid s_m\neq 0\}$ is infinite.      
        For each $m>S$ with $s_m\neq 0$, comparing the coefficients of $t^m$ in (\ref{eq_prop_lin_independent01_1}) yields
        \[
          \sum_{i=1}^{H^*} a_i\alpha_i(u)^m + m\sum_{i=1}^{H^*} \beta_i(u)\alpha_i^{m-1}(u)
          =0.
        \]
        By (\ref{eq_prop_lin_independent01_3}), the index
        $
          k:=\arg\max_{1\leq i\leq H^*}|\alpha_i(u)|
        $
        is uniquely determined.
        Dividing both sides by $\alpha_k(u)^{m-1}$, we obtain
        \[
          a_k\alpha_k(u) + m\beta_k(u)
          + \sum_{i\neq k}a_i\alpha_i(u)\left(\frac{\alpha_i(u)}{\alpha_k(u)}\right)^{m-1}
          + m\sum_{i\neq k}\beta_i(u)\left(\frac{\alpha_i(u)}{\alpha_k(u)}\right)^{m-1}
          =0.
        \]
        Since $\left|\alpha_i(u)/\alpha_k(u)\right|<1$ for $i\neq k$, letting $m\to\infty$ along integers satisfying $m>S$ and $s_m\neq 0$, we obtain $\beta_k(u)=0$.
        Substituting this back into the same relation and letting $m\to\infty$ again, we obtain $a_k=0$.
        Repeating the same argument for the remaining indices in decreasing order of $|\alpha_i(u)|$, we conclude inductively that
        $
          a_i=\beta_i(u)=0
          ~ (1\leq i\leq H^*).
        $
        Since $\beta_i(u)=c_i^\top u=0$ for all $u\in V$, and $V$ is a nonempty open set, we can choose $u_1,\ldots,u_N\in V$ forming a basis of $\mathbb{R}^N$.
        Hence $c_i^\top u_\ell=0$ for all $\ell=1,\ldots,N$, which implies that $c_i=0$ for all $i$.
        Finally, substituting these equalities into (\ref{eq_prop_lin_independent01_2}), we obtain
        $
          \sum_{1\leq |\boldsymbol{h}|\leq S}d_{\boldsymbol{h}}X^{\boldsymbol{h}}=0
          ~ q_X\text{-a.s.}
        $
        Hence the polynomial $\sum_{1\leq |\boldsymbol{h}|\leq S}d_{\boldsymbol{h}}x^{\boldsymbol{h}}$ vanishes on $U$, and therefore all coefficients $d_{\boldsymbol{h}}$ must be zero.
      \end{proof}

    \section{Preparation for the Proof of the Main Theorem}\label{sec_prepf_mainthm}
  In this section, we state the lemmas needed to prove the Main Theorem.
  In the proof, after coordinate transformations called blow-ups, we extract monomial factors from the log-likelihood ratio function $f$ and verify that the remaining factors are not zero as random variables.
  Lemma~\ref{lem_Taylor} shows that, when a monomial factor can be extracted from $f$, its square can be extracted from the Kullback--Leibler divergence $K$.
  Lemma~\ref{lem_normal_crossing_chart} shows that, on a coordinate chart after a coordinate transformation, if the remaining factor is not zero as a random variable, then $K$ has a normal-crossing form.

  In what follows, we write the coordinates as
  $
  (u,v);~
  u=(u_1,\ldots,u_p),~
  v=(v_1,\ldots,v_q)
  $.
  Let $\mu_i,\nu_j$ be nonnegative integers, and write
  $
    u^\mu v^\nu
    :=
    \prod_i u_i^{\mu_i}\prod_j v_j^{\nu_j}
  $.
\begin{lem}\label{lem_Taylor}\leavevmode\par
  Suppose that the log-likelihood ratio function $f$ can be written as
  $
    f(X|u,v)
    =
    u^\mu v^\nu \tilde f(X|u,v)
  $,
  and that at least one $\mu_i$ is positive.
  Then there exists an analytic function $\tilde K(u,v)$ such that
  \[
    K(u,v)
    =
    u^{2\mu}v^{2\nu}\tilde K(u,v)
    ,\quad
    \tilde K(0,v)
    =
    \frac12
    \mathbb{E}_X\left[\tilde f(X|0,v)^2\right].
  \]
\end{lem}

\begin{proof}
  Using the analytic function
  \[
    \psi(t):=
    \begin{cases}
      \dfrac{t+e^{-t}-1}{t^2}, & t\neq0,\\[1ex]
      \dfrac12, & t=0,
    \end{cases}
  \]
  we have
  \[
    f+e^{-f}-1=f^2\psi(f)
    =u^{2\mu} v^{2\nu} \tilde f(X|u,v)^2\psi(f).
  \]
  Also, since
  $
    e^{-f(X|u,v)}
    =
    p(X|u,v)/q(X)
  $,
  we have
  $
    \mathbb{E}_X[e^{-f(X|u,v)}]=1
  $.
  Therefore,
  \begin{equation}\label{lem_taylorNew_01}
    K(u,v)
    =
    \mathbb{E}_X
    \left[
      f(X|u,v)+e^{-f(X|u,v)}-1
    \right]
    =
    \mathbb{E}_X
    \left[
      u^{2\mu}v^{2\nu}
      \tilde f(X|u,v)^2
      \psi(f(X|u,v))
    \right].
  \end{equation}
  Since expectation and partial differentiation are interchangeable, in the Taylor expansion of $K$, every term whose degree in some $u_i$ is less than $2\mu_i$, or whose degree in some $v_j$ is less than $2\nu_j$, vanishes.
  Hence, by the analyticity of $K$, there exists an analytic function $\tilde K(u,v)$ such that
  \begin{equation}\label{lem_taylorNew_02}
    K(u,v)
    =
    u^{2\mu}v^{2\nu}\tilde K(u,v).
  \end{equation}

  We next determine $\tilde K(0,v)$.
  Differentiating the right-hand sides of \eqref{lem_taylorNew_01} and \eqref{lem_taylorNew_02} with respect to $u$ in the multi-index $2\mu$ and then setting $u=0$, and noting that at least one $\mu_i$ is positive and that $\psi(f(X|0,v))=\psi(0)=1/2$, we obtain, at every point satisfying $v^\nu\neq0$,
  \begin{equation}\label{lem_taylorNew_03}
    \tilde K(0,v)
    =
    \frac12
    \mathbb{E}_X
    \left[
      \tilde f(X|0,v)^2
    \right].
  \end{equation}
  Since both sides are continuous in $v$, this equality extends to points satisfying $v^\nu=0$.
  Therefore, \eqref{lem_taylorNew_03} holds for all $v$.
  This proves the lemma.
\end{proof}

In Lemma~\ref{lem_normal_crossing_chart}, we consider the following situation.
Take local coordinates $(\theta_1,\ldots,\theta_d)$ centered at the point under consideration, and for $d_1\leq d$, consider the blow-up $\pi$ centered at
$
  \theta_1=\cdots=\theta_{d_1}=0
$.
On the coordinate chart $U_i$ with reference coordinate $\theta_i~(1\leq i\leq d_1)$, this blow-up is defined by the coordinate transformation
  \begin{equation}\label{eq_lem_compact01}
    \begin{aligned}
      \pi|_{U_i} : ~
      \mathbb{R}^{d}\ni
      & ~(\theta_1^\prime,\ldots,\theta_{i-1}^\prime, 
      \theta_i,\theta_{i+1}^\prime,\ldots,\theta_{d_1}^\prime,
      \theta_{d_1+1},\ldots,\theta_d)\\
      & \mapsto
      (\theta_i\theta_1^\prime,\ldots,\theta_i\theta_{i-1}^\prime,
      \theta_i,\theta_i\theta_{i+1}^\prime,\ldots,\theta_i\theta_{d_1}^\prime,
      \theta_{d_1+1},\ldots,\theta_d)
      \in\mathbb{R}^d .
    \end{aligned}
  \end{equation}
We divide the coordinates after the transformation into two types.
First, we denote by $u=(u_1,\ldots,u_p)$ the parameters that approach $0$ when the original parameter approaches the origin.
These include the coordinate $\theta_i$ chosen as the reference coordinate in the blow-up, and also the coordinates $\theta_{d_1+1},\ldots,\theta_d$, which are not directly involved in the blow-up but approach $0$ near the origin.

Second, we denote by $v=(v_1,\ldots,v_q)$ the parameters newly introduced by the blow-up.
Here, the coordinates $\theta_j^\prime~(1\leq j\leq d_1,\ j\neq i)$ correspond to these parameters.
When the original parameter approaches the origin, the parameters $v$ do not necessarily approach $0$.
However, for $(\theta_1,\ldots,\theta_{d_1})\neq0$, if we choose a component $\theta_i$ with maximal absolute value as the reference coordinate, then
$
  |\theta_j^\prime|
  =
  \left|\theta_j/\theta_i\right|
  \leq1
$
for all $j\neq i$.
Thus, it is enough to regard $v$ as moving in a compact set $C$.

\begin{lem}\label{lem_normal_crossing_chart}\leavevmode\par
  Suppose that, on a coordinate chart after the blow-up $\pi$, we can write
  \[
    f(X|\pi(u,v))
    =
    u^\mu v^\nu\tilde f(X|u,v).
  \]
  Here, $\mu_i,\nu_j$ are nonnegative integers, and at least one $\mu_i$ is positive.
  Let \(C\subset\mathbb{R}^q\) be compact.
  Suppose moreover that, for every $v\in C$,
  \[
    \tilde f(X|0,v)\not\equiv0
    \quad
    \text{in }L^2(q)
  \]
  holds.
  Then, after taking a sufficiently small neighborhood of the origin in the original parameter space, we can write
  \[
    K(\pi(u,v))
    =
    u^{2\mu}v^{2\nu}\tilde K(u,v).
  \]
  Here, $\tilde K(u,v)$ is an analytic function that does not vanish on the sufficiently small coordinate neighborhood.
  Hence, $K\circ\pi$ has a normal-crossing form on this coordinate chart.
\end{lem}

\begin{proof}
  Applying Lemma~\ref{lem_Taylor} to the statistical model
  $
    p(x|\pi(u,v))
  $
  with coordinates $(u,v)$, we obtain an analytic function $\tilde K(u,v)$ such that
  \[
    K(\pi(u,v))
    =
    u^{2\mu}v^{2\nu}\tilde K(u,v)
    ,\quad
    \tilde K(0,v)
    =
    \frac12
    \mathbb{E}_X
    \left[
      \tilde f(X|0,v)^2
    \right].
  \]
  By assumption, for every $v\in C$ we have
  $
    \tilde K(0,v)>0
  $.
  Since $C$ is compact and $\tilde K(0,v)$ is continuous in $v$, there exists
  $
    m:=\min_{v\in C}\tilde K(0,v)>0
  $.
  Also, since $\tilde K(u,v)$ is continuous in $(u,v)$ and $C$ is compact,
  \[
    \tilde K(u,v)\to\tilde K(0,v)
    \quad
    (u\to0)
  \]
  uniformly for $v\in C$.
  That is, if $u$ is sufficiently close to $0$, then
  \[
    \sup_{v\in C}
    \left|
      \tilde K(u,v)-\tilde K(0,v)
    \right|
    <
    \frac{m}{2}.
  \]
  For such $u$, for every $v\in C$,
  \[
    \tilde K(u,v)
    \geq
    \tilde K(0,v)
    -
    \left|
      \tilde K(u,v)-\tilde K(0,v)
    \right|
    >
    m-\frac{m}{2}
    =
    \frac{m}{2}
    >0.
  \]
  Hence, if $u$ is sufficiently close to $0$, then
  $
    \tilde K(u,v)\neq0
  $
  for every $v\in C$.
\end{proof}

By this lemma, on each coordinate chart we only need to verify the following two points.
First, we have to extract a monomial factor from $f$.
Second, after setting $u=0$, we have to verify that the remaining factor
$\tilde f(X|0,v)$ is not zero as a random variable for the whole range of the coordinates $v$.
Once these two points are verified, Lemma~\ref{lem_normal_crossing_chart} shows that
$K\circ\pi$ has a normal-crossing form on that coordinate chart.


\section{Illustrative example}\label{sec_ex}
  To illustrate how the quantities in the Main Theorem arise in a concrete neural-network model
  and how the proof proceeds, we consider a three-layer neural network with $N=1$ input unit, $H=4$ hidden units, and $M=1$ output unit, with activation function $\sigma=\tanh$.
  In this case, using
  $A=(\theta_1+1~a_1~a_2~a_3)\in\mathbb{R}^{1\times 4}$ and
  $B=(\theta_2+1~b_1~b_2~b_3)^\top\in\mathbb{R}^{4\times 1}$, we can write
  \begin{align*}
    Y &= A\sigma(BX) + \mathcal{N} \\
      &= (\theta_1+1)\sigma((\theta_2+1)X) + a_1\sigma(b_1X) + a_2\sigma(b_2X) + a_3\sigma(b_3X) + \mathcal{N} \\
      &=: \phi(X|\theta,a,b) + \mathcal{N}.
  \end{align*}
  Assume that the true distribution is given by
  $
    Y=\sigma(X)+\mathcal{N}.
  $
  In particular, the number of true hidden units is $H^*=1$.
  We evaluate the local learning coefficient at the realization parameter
  $
    (\theta_1,\theta_2,a_1,a_2,a_3,b_1,b_2,b_3)=0.
  $
  Let $Z:=Y-\sigma(X)$.
  Then
  \begin{equation}\label{eq_ex_pre_01}
    f
    =-Z\left\{\phi(X|\theta,a,b)-\sigma(X)\right\}
    +\frac{1}{2}\left\{\phi(X|\theta,a,b)-\sigma(X)\right\}^2.
  \end{equation}
  Moreover,
  \begin{align*}
    \phi(X|\theta=0,a,b) - \sigma(X)
    &=a_1\sigma(b_1X) + a_2\sigma(b_2X) + a_3\sigma(b_3X)\\
    &=\sum_{s=1}^\infty \frac{\sigma^{(2s-1)}(0)}{(2s-1)!}
    \left(
      a_1b_1^{2s-1} + a_2b_2^{2s-1} + a_3b_3^{2s-1}
    \right)X^{2s-1}\\
    &=\sum_{s=1}^\infty g_s(a,b)\,R_s(X),
  \end{align*}
  where
  $
    g_s(a,b):=a_1b_1^{2s-1} + a_2b_2^{2s-1} + a_3b_3^{2s-1},
    ~
    R_s(X):=\frac{\sigma^{(2s-1)}(0)}{(2s-1)!}X^{2s-1}.
  $
  Clearly, $g_s(a,b)$ is analytic and homogeneous of degree $2s-1$ in $b$.
  Let $Z_s:=-Z\,R_s(X)$.
  Then, by \eqref{eq_ex_pre_01},
  \begin{align*}
    f(X|\theta=0,a,b)
    &=\sum_{s=1}^\infty g_s(a,b)Z_s 
    + \textstyle\frac{
        g_1^2(a,b)R_1^2(X) + 2g_1(a,b)g_2(a,b)R_1(X)R_2(X)+\cdots
    }
    {2}\\
    &=\sum_{s=1}^\infty g_s(a,b)Z_s + \sum_{s=1}^\infty h_s(a,b)W_s,
  \end{align*}
  where $W_s$ are random variables and $h_s(a,b)\in$ ideal $(g_1(a,b),g_2(a,b),\ldots)^2$.

  We verify that this model satisfies assumptions (i)--(iv) of the Main Theorem with
  \[
    (r,\alpha,\beta,\gamma,m_s,n_s)
    =(2,3,3,\infty,2s-1,1),
    \qquad
    L=3,\quad n_1^\ast=n_2^\ast=n_3^\ast=1.
  \]
  Conditions (i) and (iv) have already been verified above.

  Condition (ii) is seen from
  \[
    \left.
      \frac{\partial(g_1,g_2,g_3)}{\partial(a_1,a_2,a_3)}
    \right|_{a=0}
    =
    \begin{pmatrix}
      b_1 & b_1^3 & b_1^5\\
      b_2 & b_2^3 & b_2^5\\
      b_3 & b_3^3 & b_3^5
    \end{pmatrix},
  \]
  whose determinant is
  $
    b_1b_2b_3(b_1^2-b_2^2)(b_2^2-b_3^2)(b_3^2-b_1^2).
  $
  Moreover, by the linear independence in \eqref{eq_sing01_01},
  $
    \{\sigma(X),\sigma^\prime(X)X,X,X^3,X^5\}
  $
  is linearly independent, and hence condition (iii) follows.

  Therefore, by the Main Theorem, we have 
  $
    \lambda_P \leq 11/6.
  $  
  We now verify this by following, in this concrete example, the same procedure as in the proof of the Main Theorem.
  The Taylor expansion of the log-likelihood ratio function $f$ at $(\theta,a,b)=0$ can be written, 
  by separating the terms involving $\theta$ from those not involving $\theta$, as
  \begin{align*}
    f(X|\theta,a,b)
    = f(X|\theta=0,a,b)
      + \sum_{j\geq 1}\sum_{l\geq 0}
      h_{j,l}(\theta,a,b)W^\prime_{j,l}.
  \end{align*}
  Here, $W_{j,l}^\prime$ are random variables, 
  and $h_{j,l}(\theta,a,b)$ are homogeneous polynomials of degree $j$ in $\theta$ and degree $l$ in $b$, with arbitrary dependence on $a$.
  In particular, the term corresponding to $(j,l)=(1,0)$ is
  \begin{gather*}
    \left.
      \theta_1\frac{\partial f}{\partial \theta_1}
    \right|_{(\theta,b)=0}
    +\left.
      \theta_2\frac{\partial f}{\partial \theta_2}
    \right|_{(\theta,b)=0}
    =\theta_1D_1 + \theta_2D_2,~
    D_1 := -Z\sigma(X),~ D_2 := -Z\sigma^\prime(X)X.
  \end{gather*}
  Therefore, $f$ can be written as
  \begin{align*}
    f
    =&\ \theta_1D_1 + \theta_2D_2
    + \sum_{s\geq 1} g_s(a,b)Z_s
    + \sum_{s\geq 1} h_s(a,b) W_s 
    + \sum_{\substack{j+l \geq 2 \\ j \geq 1,\, l \geq 0}}
      h_{j,l}(\theta,a,b)W^\prime_{j,l}.
  \end{align*}
  As we have already confirmed, the random variables
  $
    D_1,\ D_2,\ Z_1,\ Z_2\ \text{and } Z_3
  $
  are linearly independent.
  Thus, in a neighborhood $W$ of $(\theta,a,b)=0$, 
  we perform coordinate transformations of the parameters in four steps in order to obtain a normal crossing form of $K(\theta,a,b)$.
  Figure~\ref{fig_ex_mainthm01} displays the full branching structure.
  Below, we spell out only the branches needed to show why the coordinate change in Step~2 is required
  and how the minimum candidate value arises.
  \begin{figure}
    \centering
    \begin{tikzpicture}[
      x=0.6cm, y=1cm, 
      >=Stealth,
      nonterm/.style    ={circle,draw=black,line width=1.0pt,minimum size=5.5mm,inner sep=0pt},
      nontermBig/.style ={circle,draw=black,line width=1.2pt,minimum size=10mm, inner sep=0pt},
      term/.style       ={circle,draw=blue, line width=1.0pt,minimum size=5.5mm,inner sep=0pt},
      termBig/.style    ={circle,draw=blue, line width=1.2pt,minimum size=10mm, inner sep=0pt},
      halo/.style       ={circle,draw=red,  dashed,line width=1.1pt,minimum size=10mm, inner sep=0pt}
    ]
    
    \def\TransLabelFont{\scriptsize} 
    \tikzset{
      translabel/.style={font=\TransLabelFont, inner sep=1pt, fill=white}
    }
    
    \node[nontermBig] (S0)  at (0,0)   {};
    \node[termBig]    (T1)  at (0,-2)  {}; 
    
    \node[nontermBig] (S2)  at (3,0)   {}; 
    
    \node[nonterm]    (A0)  at (5.5,0) {}; 
    \node[halo]        at (5.5,0) {};       
    
    \node[term]       (A0d) at (5.5,-2) {};
    \node[nonterm]    (B0)  at (8,0)    {};
    \node[term]       (B0d) at (8,-2)   {};
    \node[nonterm]    (C0)  at (10.5,0) {}; 
    
    \node[term]       (C0d) at (10.5,-2) {};
    \node[nonterm]    (D0)  at (13,0)    {};
    \node[term]       (D0d) at (13,-2)   {};
    \node[nonterm]    (E0)  at (15.5,0)  {}; 
    
    \node[term]       (E0d) at (15.5,-2) {};
    \node[term]       (F0)  at (18,0)    {}; 
    
    \draw[->,line width=1.0pt] (S0) -- node[left,  translabel] {CT1}     (T1);
    \draw[->,line width=1.0pt] (S0) -- node[above, translabel] {CT2}     (S2);
    
    \draw[->,line width=1.0pt] (S2) -- (A0); 
    
    \draw[->,line width=1.0pt] (A0) -- node[left,  translabel] {CT3--CT4} (A0d);
    \draw[->,line width=1.0pt] (A0) -- node[above, translabel] {CT5}      (B0);
    \draw[->,line width=1.0pt] (B0) -- node[left,  translabel] {CT3--CT4} (B0d);
    \draw[->,line width=1.0pt] (B0) -- node[above, translabel] {CT5}      (C0);
    
    \draw[->,line width=1.0pt] (C0) -- node[left,  translabel] {CT6--CT7'} (C0d);
    \draw[->,line width=1.0pt] (C0) -- node[above, translabel] {CT8}      (D0);
    \draw[->,line width=1.0pt] (D0) -- node[left,  translabel] {CT6--CT7'} (D0d);
    \draw[->,line width=1.0pt] (D0) -- node[above, translabel] {CT8}      (E0);
    
    \draw[->,line width=1.0pt] (E0) -- node[left,  translabel] {CT9}      (E0d);
    \draw[->,line width=1.0pt] (E0) -- node[above, translabel] {CT10}     (F0);
    
    \def\yStep{1.25}
    \def\yVBot{0.70}
    \def\yVTop{1.45}
    
    \foreach \N in {S0,S2,A0,C0,E0,F0}{
      \draw[dashed,line width=0.9pt] ($(\N)+(0,\yVBot)$) -- ($(\N)+(0,\yVTop)$);
    }
    
    \draw[dotted,<->,line width=1.0pt] ($(S0)+(0,\yStep)$) -- node[above]{Step1}   ($(S2)+(0,\yStep)$);
    \draw[dotted,<->,line width=1.0pt] ($(S2)+(0,\yStep)$) -- node[above]{Step2}   ($(A0)+(0,\yStep)$);
    \draw[dotted,<->,line width=1.0pt] ($(A0)+(0,\yStep)$) -- node[above]{Step3-1} ($(C0)+(0,\yStep)$);
    \draw[dotted,<->,line width=1.0pt] ($(C0)+(0,\yStep)$) -- node[above]{Step3-2} ($(E0)+(0,\yStep)$);
    \draw[dotted,<->,line width=1.0pt] ($(E0)+(0,\yStep)$) -- node[above]{Step4}   ($(F0)+(0,\yStep)$);
    
    \end{tikzpicture}
    \caption{Procedure of the coordinate transformations $\pi$ carried out in this example}
    \caption*
    {\footnotesize The red dashed ring at the endpoint of Step2 indicates that the subsequent operations are performed only on a part (a restricted neighborhood) around that point.}
    \label{fig_ex_mainthm01}
  \end{figure}

  \subsection*{Step 1: Blow-up centered at $\{(\theta,b)=0\}$}
    We perform each of the following two coordinate transformations, CT1 and CT2, once.
    \begin{description}
      \item[CT1]
        For $1\leq t\leq 2$,
        $\{\theta_j\rightarrow \theta_t\theta_j',~
        b_1\rightarrow \theta_t b_1^{\prime},
        b_2\rightarrow \theta_t b_2^{\prime},
        b_3\rightarrow \theta_t b_3^{\prime}
        \mid 1\leq j \leq 2,~ j\neq t\}$.
      \item[CT2]
        For $1\leq t\leq 3$,
        $\{\theta_1\rightarrow b_t\theta_1^\prime,
        \theta_2\rightarrow b_t\theta_2^\prime,~
        b_k \rightarrow b_tb_k^{\prime} \mid 1\leq k \leq 3,~ k\neq t\}$.
    \end{description}
    After one application of CT1, a normal crossing is obtained
    and the corresponding candidate value is $5/2$.
    We therefore focus on the CT2 branch,
    where a normal crossing is not yet obtained.

  \paragraph*{Applying CT2 once}
    Under the coordinate transformation
    $
    \text{CT2: }\pi=\{
    \theta_1\rightarrow b_1\theta_1^\prime,~
    \theta_2\rightarrow b_1\theta_2^\prime,~
    b_2\rightarrow b_1 b_2^{\prime},~
    b_3\rightarrow b_1 b_3^{\prime}
    \},
    $
    we can write
    \begin{align}\label{eq_ex_step1}
      f
      =&\ b_1
      \Bigl\{
        \theta_1^\prime D_1+\theta_2^\prime D_2
        + \sum_{s\geq 1}b_1^{2s-2}g_s(a,1,b_2^\prime,b_3^\prime)Z_s
        + b_1\sum_{s\geq 1}
        \tilde{h}_s(a,b_1,b_2^\prime,b_3^\prime)W_s
        \notag\\
        &\quad~
        + b_1\sum_{\substack{j+l \geq 2 \\ j \geq 1,\, l \geq 0}}
          b_1^{j+l-2}
          h_{j,l}(\theta_1^\prime,\theta_2^\prime,a_1,a_2,a_3,1,b_2^\prime,b_3^\prime)
          W^\prime_{j,l}
      \Bigr\}
      \notag\\
      =&\ b_1\tilde{f}(X),
      \qquad
      \left.\tilde{f}(X)\right|_{b_1=0}
      =\theta_1^\prime D_1+\theta_2^\prime D_2+(a_1+a_2b_2^\prime+a_3b_3^\prime)Z_1.
    \end{align}
    Therefore, by Lemma~\ref{lem_Taylor},
    \begin{gather*}
      K(\theta,a,b) 
      = b_1^2\tilde{K}(\theta^\prime,a,b_1,b_2^\prime,b_3^\prime),\quad
      \tilde{K}(\theta^\prime,a,b_1=0,b_2^\prime,b_3^\prime)=\frac{1}{2}\mathbb{E}[\tilde{f}(X)^2|_{b_1=0}]
    \end{gather*}
    holds.
    Here, $\tilde{h}_s\in$ ideal
    $
    \left(
      b_1^{2s-2}g_s(a,1,b_2^\prime,b_3^\prime)\mid s\geq 1
    \right)^2.
    $
    If
    $
    (\theta_1^\prime,\theta_2^\prime,g_1(a,1,b_2^\prime,b_3^\prime))=(0,0,0),
    $
    then \(\tilde f(X)\equiv0\) in \(L^2(q)\), so the resulting form is not a normal crossing.

  \subsection*{Step 2: Coordinate transformation $a\mapsto a^{\prime}$}
    The coordinate transformation
    \[
      \begin{cases}
        a_1^\prime := g_1(a,1,b_2^\prime,b_3^\prime)
        = a_1 + a_2b_2^\prime + a_3b_3^\prime,\\
        a_2^\prime := g_2(a,1,b_2^\prime,b_3^\prime)
        = a_1 + a_2b_2^{\prime 3} + a_3b_3^{\prime 3},\\
        a_3^\prime := g_3(a,1,b_2^\prime,b_3^\prime)
        = a_1 + a_2b_2^{\prime 5} + a_3b_3^{\prime 5}
      \end{cases}
    \]
    maps $a=0$ to $a^\prime=0$ and is analytically invertible on the nonempty open set
    \[
      \left\{
        (b_2^\prime,b_3^\prime)\ \middle|\
        b_2^{\prime}b_3^{\prime}(b_2^{\prime 2}-1)(b_3^{\prime 2}-1)(b_2^{\prime 2}-b_3^{\prime 2})\neq0
      \right\}.
    \]
    In what follows, we restrict attention to this open set and compute the real log canonical threshold there.
    Since we restrict to a subset of a neighborhood, the resulting real log canonical threshold gives an upper bound for the exact value.
    
    Since
    \[
    \begin{pmatrix}
      a_1 \\ a_2 \\ a_3
    \end{pmatrix}
    =
    \begin{pmatrix}
      1 & b_2^\prime & b_3^\prime\\
      1 & b_2^{\prime 3} & b_3^{\prime 3}\\
      1 & b_2^{\prime 5} & b_3^{\prime 5}
    \end{pmatrix}^{-1}
    \begin{pmatrix}
      a_1^\prime \\ a_2^\prime \\ a_3^\prime
    \end{pmatrix},
    \]
    we have
    $
      g_s(a,1,b_2^\prime,b_3^\prime)
      \in
      \text{ideal}\left(a_1^\prime,a_2^\prime,a_3^\prime\right)
      ~\text{for all }s\geq 4.
    $
    Moreover,
    \begin{align*}
      \text{ideal}
      \left(
        b_1^{2s-2}g_s(a,1,b_2^\prime,b_3^\prime)\mid s\geq 1
      \right)
      &=
      \text{ideal}
      \left(
        a_1^\prime,
        b_1^2a_2^{\prime},
        b_1^4a_3^{\prime},
        b_1^6g_4(a,1,b_2^\prime,b_3^\prime),
        \ldots
      \right)\\
      &=
      \text{ideal}
      \left(
        a_1^\prime,
        b_1^2a_2^{\prime},
        b_1^4a_3^{\prime}
      \right).
    \end{align*}
    Hence, \eqref{eq_ex_step1} can be rewritten as
    \begin{equation}\label{eq_ex_step2}
    \begin{aligned}
      f =&\, b_1
      \Bigl\{
        \theta_1^\prime D_1+\theta_2^\prime D_2
        + a_1^\prime Z_1
        + b_1^2a_2^{\prime}Z_2
        + b_1^4a_3^{\prime}Z_3
        + b_1^6\sum_{s\geq 4}b_1^{2s-8}
        k_sZ_s\\
        &\quad~+ b_1\sum_{s\geq 1}
        \tilde{h}_sW_s
        + b_1\sum_{\substack{j+l \geq 2 \\ j \geq 1, \, l \geq 0}}
          b_1^{j+l-2}
          h_{j,l}(\theta_1^\prime,\theta_2^\prime)
          W^\prime_{j,l}
      \Bigr\}.
    \end{aligned}
    \end{equation}
    Here,
    $
      k_s \in \text{ideal}\left(a_1^\prime,a_2^\prime,a_3^\prime\right),
    $
    and
    $
      \tilde{h}_s \in \text{ideal}
      \left(
        a_1^\prime,
        b_1^2a_2^{\prime},
        b_1^4a_3^{\prime}
      \right)^2.
    $
    For simplicity, we continue to denote by $\tilde{h}_s$ and $h_{j,l}$ the functions after this coordinate transformation.
    Although $h_{j,l}$ are still functions of
    $
    \{a_1^\prime,a_2^\prime,a_3^\prime, b_2^\prime, b_3^\prime\},
    $
    we omit this dependence for clarity.
    We also note that, after the coordinate transformation, $h_{j,l}$ remains homogeneous of degree $j$ in $\theta^\prime$, but is not necessarily homogeneous of degree $l$ in $b^\prime$.

\subsection*{Step 3: Blow-up}

    We apply the following three coordinate transformations, CT3--CT5, according to the procedure shown in Figure~\ref{fig_ex_mainthm01}.
    \begin{description}
      \item[CT3]
        For $1\leq t\leq 2$,
        $\{\theta_j^\prime \rightarrow \theta_t^\prime \theta_j^{\prime\prime},~
        a_1^\prime \rightarrow \theta_t^\prime a_1^{\prime\prime},~
        b_1 \rightarrow \theta_t^\prime b_1^\prime
        \mid 1\leq j\leq 2,~ j\neq t\}$.
      \item[CT4]
        $\{\theta_1^\prime \rightarrow a_1^\prime \theta_1^{\prime\prime},~
        \theta_2^\prime \rightarrow a_1^\prime \theta_2^{\prime\prime},~
        b_1 \rightarrow a_1^\prime b_1^\prime\}$.
      \item[CT5]
        $\{\theta_1^\prime \rightarrow b_1\theta_1^{\prime\prime},~
        \theta_2^\prime \rightarrow b_1\theta_2^{\prime\prime},~
        a_1^\prime \rightarrow b_1 a_1^{\prime\prime}\}$.
    \end{description}
   The charts corresponding to CT3 and CT4 already yield normal crossings,
    with candidate value $8/4$.
    The same holds for the charts obtained by applying CT5 once and then CT3 or CT4.
    We therefore describe in detail only the case in which CT5 is applied twice,
    which yields the smallest candidate value in this example.   

    \paragraph*{Applying CT5 twice}
    Under the coordinate transformation
    $
    \pi=\{\theta_1^\prime \rightarrow b_1^2\theta_1^{\prime\prime},~
    \theta_2^\prime \rightarrow b_1^2\theta_2^{\prime\prime},~
    a_1^\prime \rightarrow b_1^2 a_1^{\prime\prime}\},
    $
    \eqref{eq_ex_step2} can be rewritten as
    \begin{equation}\label{eq_ex_step3-1}
    \begin{aligned}
      f
      =&\, b_1^3
      \Biggl\{
        \theta_1^{\prime\prime} D_1+\theta_2^{\prime\prime} D_2
        + a_1^{\prime\prime} Z_1
        + a_2^\prime Z_2
        + b_1^2 a_3^\prime Z_3
        + b_1^4\sum_{s\geq 4} b_1^{2s-8} k_s^{(2)} Z_s\\
        &\quad
        + b_1^3\sum_{s\geq 1} \tilde{h}_s^{(2)} W_s
        + b_1\sum_{\substack{j+l \geq 2 \\ j \geq 1,\, l \geq 0}}
          b_1^{(j+l-2)+2(j-1)}
          h_{j,l}^{(2)}(\theta_1^{\prime\prime},\theta_2^{\prime\prime})
          W_{j,l}^\prime
      \Biggr\}\\
      =&\, b_1^3\tilde{f}(X),
      \qquad
      \left.\tilde{f}\right|_{b_1=0}
      =\theta_1^{\prime\prime} D_1+\theta_2^{\prime\prime} D_2
      + a_1^{\prime\prime} Z_1
      + a_2^\prime Z_2.
    \end{aligned}
    \end{equation}
    Here,
    $
      k_s^{(2)} \in \text{ideal}\left(b_1^2 a_1^{\prime\prime},a_2^\prime,a_3^\prime\right),
    $
    and
    $
      \tilde{h}_s^{(2)} \in \text{ideal}\left(a_1^{\prime\prime},a_2^\prime,b_1^2 a_3^\prime\right)^2.
    $
    A normal crossing is obtained at any point $Q$ satisfying
    $
    (\theta_1^{\prime\prime},\theta_2^{\prime\prime},a_1^{\prime\prime},a_2^\prime,b_1)\neq 0.
    $
    Since the Jacobian of this coordinate transformation is $b_1^{10}$, we obtain
    \[
      \inf_Q \min_j \frac{h_j^{(Q)}+1}{k_j^{(Q)}}
      =\frac{11}{6}.
    \]
    Hence, it remains to consider only points satisfying
    $
    (\theta_1^{\prime\prime},\theta_2^{\prime\prime},a_1^{\prime\prime},a_2^\prime,b_1)=0.
    $
    

  The remaining cases are treated in the same manner.
  Their candidate values are summarized below.

  \begin{center}
  \small
  \begin{tabular}{@{}ccccc@{}}
    \toprule
    Stage & Step 1 & Step 3-1 & Step 3-2 & Step 4 \\
    \midrule
    Candidate values
      & $5/2$
      & $8/4,\ 11/6$
      & $15/8,\ 19/10$
      & $5/2$ \\
    \bottomrule
  \end{tabular}
  \end{center}

  Therefore,
  \[
    \lambda_P
    \leq \min
    \left\{
      \frac{5}{2},\frac{8}{4},\frac{11}{6},\frac{15}{8},\frac{19}{10},\frac{5}{2}
    \right\}
    =\frac{11}{6}.
  \]

\section{Detailed Proof of the Main Theorem}\label{sec_DetailPf_mainthm}
  We denote
  $$
  D_i(a):=
  \left.
    \frac{\partial f}{\partial \theta_i}
  \right|_{(\theta,b)=0},\quad
  X_i:=
  D_i(a=0)
  \quad(i=1,\ldots,r).
  $$
  Fixing $a$ in a neighborhood of $0$, the Taylor expansion of $f$ around $(\theta,b)=0$ can be written as
  \begin{align*}
    f &= 
    f(X|\theta=0,a,b) 
    + \sum_{j\geq1}\sum_{l\geq0}h_{j,l}(\theta,a,b)W_{j,l}^\prime\\
    &=
    \sum_{i=1}^r \theta_iD_i(a) 
    + \sum_{s=1}^\gamma \sum_{n=1}^{n_s} g_{s,n}(a,b)Z_{s,n} + \sum_{s=1}^\infty h_s(a,b)W_s 
    + \sum_{\substack{j+l\geq2\\j\geq1,l\geq0}}h_{j,l}W_{j,l}^\prime.
  \end{align*}
  Here $W_{j,l}^\prime$ are random variables, and $h_{j,l}(\theta,a,b)$ are homogeneous polynomials of degree $j$ in $\theta$ and of degree $l$ in $b$ (with arbitrary degree in $a$).
  By assumption~(iii), $X_1,\ldots,X_r,Z_{1,1},\ldots,Z_{L,n_L}$ are linearly independent.

  \subsection*{Step 1: Blow-up centered at $\{(\theta,b)=0\}$}
    We perform the following four types of coordinate transformations, CT1--CT4,
    \footnote{In the terminology of algebraic geometry, CT1 and CT2 are blow-ups centered at the subvariety $\{(\theta,a,b)\mid \theta=b=0\}$, 
    whereas CT3 and CT4 are blow-ups centered at the subvariety $\{(\theta,a,b)\mid \theta=b_t=0\}$ when the exceptional divisor arising from CT2 is given by $\{b_t=0\}$.}
    according to the procedure in Figure~\ref{fig3}.
    (If $r=0$, only CT2 is used; we do not mention this case explicitly below.)
    \begin{description}
      \item[CT1]
        For $1\leq t\leq r$,
        $\{\theta_j\rightarrow \theta_t\theta_j'~(1\leq j \leq r,\ j\neq t),~
        b_k\rightarrow \theta_t b_k^{\prime}~(1\leq k\leq \beta)\}$.
      \item[CT2] 
        For $1\leq t\leq \beta$,
        $\{\theta_j\rightarrow b_t\theta_j'~(1\leq j \leq r),~
        b_k \rightarrow b_tb_k^{\prime}~(1\leq k \leq \beta,\ k\neq t)\}$.
      \item[CT3] 
        For the index $t$ chosen in CT2, and for $1\leq s\leq r$,
        $\{\theta_j^{\prime}\rightarrow \theta_s\theta_j^{\prime\prime}~(1\leq j \leq r,\ j\neq s),~
        b_t^{\prime}\rightarrow \theta_s b_t^{\prime\prime}\}$.
      \item[CT4] 
        For the index $t$ chosen in CT2,
        $\{\theta_j^{\prime}\rightarrow b_t\theta_j^{\prime\prime}~(1\leq j\leq r)\}$.
    \end{description}

    We first apply CT1 and CT2.
    In the former case, normal crossings are obtained, whereas in the latter case they are not, and we subsequently apply CT3 and CT4.
    Again, in the former case normal crossings are obtained, whereas in the latter case they are not.
    Repeating this process $m_1-1$ more times (so that it is performed $m_1$ times in total), we obtain normal crossings after applying CT3.

    \subsubsection*{Performing CT1 once}
      We may assume without loss of generality that $t=1$.
      If we perform
      CT1:$\pi=\{\theta_j\rightarrow \theta_1\theta_j'~(2\leq j \leq r),~
      b_k\rightarrow \theta_1 b_k^{\prime}~(1\leq k\leq \beta)\}$,
      then we can write
      \begin{align*}
        f =& \theta_1
        \left\{
        D_1(a)+\sum_{i=2}^r\theta_i^{\prime}D_i(a) + \theta_1^{m_1-1}\sum_{n=1}^{n_1}g^\prime_{1,n}(a,b^{\prime})Z_{1,n} + \theta_1\times(\cdots)
        \right\}\\
         =& \theta_1\tilde{f}(X)   
         ~;~
         \left.\tilde{f}(X)\right|_{(\theta_1,a)=0}
         =X_1+ \sum_{i=2}^r \theta_i^{\prime}X_i + \sum_{n=1}^{n_1}g^\prime_{1,n}(0,b^{\prime})Z_{1,n}\times\mathbf{1}_{m_1=1}
         \not\equiv 0
        \quad\text{in }L^2(q).
      \end{align*}
      Thus, applying Lemma~\ref{lem_normal_crossing_chart} with
      $
      u=(\theta_1,a)\text{ and }v=(\theta_2^\prime,\ldots,\theta_r^\prime,b^\prime),
      $
      we see that a normal crossing form has been obtained.
      We omit this verification in the remaining cases.
      Since the Jacobian of this coordinate transformation is $\theta_1^{r+\beta-1}$, we obtain
        \[
        \inf_{Q}\min_{j}{\frac{h_j^{(Q)}+1}{k_j^{(Q)}}}=\frac{r+\beta}{2}.
        \]
        
    \subsubsection*{After performing CT2 once, performing CT4 $k(\leq m_1-2)$ times, and finally performing CT3}
      By the same argument as in Kurumadani~\cite{kurumadani2025b} and above, we see that normal crossings are obtained, and we obtain the following.
      \[
      \inf_{Q}\min_{j}{\frac{h_j^{(Q)}+1}{k_j^{(Q)}}}=\frac{r}{2}+\frac{\beta}{2(k+2)}.
      \]

    \subsubsection*{After performing CT2 once and then performing CT4 $m_1-1$ times}
      Take an arbitrary point $\hat b=(\hat b_1,\ldots,\hat b_\beta)\neq0$ satisfying the rank condition in Condition~(ii).
      After relabeling the indices, we may assume that $\hat b_1\neq0$.
      We apply the coordinate transformation
      $\pi=\{\theta_j\rightarrow b_1^{m_1}\theta_j^\prime~(1\leq j \leq r),~
      b_k \rightarrow b_1b_k^{\prime}~(2\leq k \leq \beta)\}$
      (whose Jacobian is $b_1^{m_1r+\beta-1}$).
      We write $b_2^{\prime},\ldots,b_\beta^\prime$ simply as $\tau$, and set
      $
        T:=(\hat b_2/\hat b_1,\ldots,\hat b_\beta/\hat b_1)
      $.
      In what follows, we work on the part where the transformed coordinate $\tau$ is sufficiently close to $T$.
      Using $g_{s,n}=b_1^{m_s}\tilde{g}_{s,n}(a,b_1,\tau)$,
      where $\tilde{g}_{s,n}$ is analytic, 
      we can write
      \begin{equation}\label{eq_pf_mainthm_step1}
      \begin{aligned}
        f =& b_1^{m_1}
        \Biggl\{
          \sum_{i=1}^r \theta_i^{\prime} D_i(a)
          + \sum_{s=1}^\gamma\sum_{n=1}^{n_s} b_1^{m_s-m_1}\tilde{g}_{s,n}(a,b_1,\tau)Z_{s,n} \\
          &\quad\quad
          + b_1^{m_1} \sum_{s=1}^\infty \tilde{h}_s(a,b_1,\tau) W_s
          + \sum_{\substack{j+l\geq 2\\ j\geq1, l\geq 0}}b_1^{m_1(j-1)+l}h_{j,l}(\theta^{\prime})W_{j,l}^\prime
        \Biggr\}.
      \end{aligned}
      \end{equation}
      Here, we have $\tilde{h}_s(a,b_1,\tau)\in$ 
      ideal
      $
      \left(
        b_1^{m_s-m_1}\tilde{g}_{s,n}(a,b_1,\tau)
        \mid
        1\leq s\leq \gamma,~ 1\leq n\leq n_s
      \right)^2
      $,
      and $h_{j,l}(\theta^{\prime})$ is a homogeneous polynomial of degree $j$ in $\theta^{\prime}$.
      In this coordinate neighborhood alone, normal crossings have not yet been obtained.

  \subsection*{Step 2: Coordinate transformation $a\mapsto a^{\prime}$}
    First, we consider the case where
    $\sum_{s=1}^\gamma n_s\geq \alpha$. In this case,
    note that $\sum_{s=1}^L n_s^\ast = \alpha$ holds.

    For the $T$ fixed in Step~1, the rank condition in Condition~(ii) holds at
    $(a,b_1,\tau)=(0,1,T)$.
    We define the coordinate transformation $(a,b_1,\tau)\mapsto (a^\prime, b_1,\tau)$ by
    $a_{s,n}^\prime:=\tilde{g}_{s,n}(a,b_1,\tau)$,
    $(1\leq s\leq L, 1\leq n\leq n_s^*)$.
    By condition (i), we have $\bar{g}_{s,n}(0,b_1,\tau)=0$,
    and hence this coordinate transformation satisfies $(0,0,T)\mapsto(0,0,T)$.
    Furthermore, since $\tilde{g}_{s,n}(a,0,\tau)=\bar{g}_{s,n}(a,1,\tau)$, we have
    $
    \left.\frac{\partial\tilde{g}_{s,n}}{\partial a}\right|_{(a,b_1,\tau)=(0,0,T)}
    =\left.\frac{\partial\bar{g}_{s,n}}{\partial a}\right|_{(a,b_1,\tau)=(0,1,T)}
    $,
    and this is a nonsingular matrix by condition (ii).
    Therefore,
    since the Jacobian of this coordinate transformation at the point $(a,b_1,\tau)=(0,0,T)$ is a nonsingular matrix,
    the inverse function theorem implies that this coordinate transformation is locally an analytic isomorphism
    near the chosen point.
    In what follows, we work on the coordinate neighborhood corresponding to this $T$.
    The normal crossing obtained on this chart gives the candidate value
    used in the upper-bound calculation for $\lambda_P$.
    Hereafter, using an analytic map $\varphi$, we write
    $a = \varphi(a^\prime,b_1,\tau)$.
    Note that this satisfies $\varphi(a^\prime=0,b_1,\tau)=0$.
 
    Using the index set
    $
    \mathcal{S}:=\{(s,n)\mid s=L, n_L^*+1\leq n\leq n_L\}
    \cup
    \{(s,n)\mid L+1\leq s\leq \gamma, 1\leq n\leq n_s\}
    $,
    we can write (\ref{eq_pf_mainthm_step1}) after the coordinate transformation as
    \begin{equation} \label{eq_pf_mainThm_01}
    \begin{aligned}
      f =& b_1^{m_1}
      \Biggl\{
        \sum_{i=1}^r \theta_i^{\prime}D_i(\varphi(a^\prime,b_1,\tau)) 
        + \sum_{s=1}^L 
        \sum_{n=1}^{n_s^\ast}
        b_1^{m_s-m_1}
        a_{s,n}^{\prime}Z_{s,n}
        + b_1^{m_{L}-m_1}\sum_{(s,n)\in\mathcal{S}}k_{s,n}Z_{s,n}
        \\
      &\quad\quad
      + b_1^{m_1}\sum_{s=1}^\infty \tilde{h}_sW_s
      + \sum_{\substack{j+l\geq 2\\j\geq 1,\, l\geq 0}}
        b_1^{m_1(j-1)+l}h_{j,l}(\theta^\prime)W_{j,l}^\prime
      \Biggr\}.
    \end{aligned}
    \end{equation}
    Here,
    since
    $
    \left.\tilde{g}_{s,n}(\varphi(a^\prime,b_1,\tau),b_1,\tau)\right|_{a^\prime=0}
    =0
    $, we obtain
    \begin{equation*}
      (s,n)\in\mathcal{S},~
      \tilde{g}_{s,n}(\varphi(a^\prime,b_1,\tau),b_1,\tau)
      \in\text{ideal }(a_{1,1}^\prime,\ldots,a_{L,n_L^\ast}^\prime),  
    \end{equation*}
    and
    $k_{s,n}\in$ ideal $(a_{s,n}^\prime\mid 1\leq s\leq L, 1\leq n\leq n_s^*)$,
    $\tilde{h}_s\in$ ideal
    $(b_1^{m_s-m_1}a_{s,n}^\prime\mid 1\leq s\leq L, 1\leq n\leq n_s^\ast)^2$.
    Hereafter, unless there is no risk of confusion, we abbreviate $D_i(\varphi(a^\prime,b_1,\tau))$ as $D_i(a^\prime)$.
    Note that, with this notation, $D_i(a^\prime=0)=X_i$ holds.

  When $\sum_{s=1}^\gamma n_s< \alpha$, we have $n_s^*=n_s$.
  In this case, we choose $\sum_{s=1}^\gamma n_s$ of the parameters $a$ so that, by assumption~(ii), the Jacobian matrix
  $
    \left.
    \frac{\partial (\bar{g}_{1,1},\ldots,\bar{g}_{L,n_L})}{\partial (a_{i_1},\ldots,a_{i_{\sum n_s}})}
    \right|_{a=0}
  $
  is nonsingular, and then apply the same coordinate transformation as above.
  In this way, we obtain (\ref{eq_pf_mainThm_01}).

  \subsection*{Step 3: Blow-up using $a^{\prime}$}
    Using (\ref{eq_pf_mainThm_01}), we compute the candidate value at the point
    $
      (\theta^{\prime},a^{\prime},b_1,\tau)=(0,0,0,T),
    $
    where $T$ is arbitrary subject to the rank condition above.

    \subsubsection*{Step 3-1: First coordinate transformation}
    We perform the following three types of coordinate transformations, CT5--CT7,
    \footnote{In the terminology of algebraic geometry, these coordinate transformations correspond to the blow-up centered at the subvariety
    $\{(\theta^{\prime},a^{\prime},b_1,\tau)\mid \theta^{\prime}=a_{1,\cdot}^{\prime}=b_1=0\}$.}
    according to the procedure in Figure~\ref{fig3}.
    \begin{description}
      \item[CT5]
        For $1\leq t\leq r$,
        $\{\theta_j^{\prime}\rightarrow \theta_t^{\prime}\theta_j^{\prime\prime}~(1\leq j \leq r,\ j\neq t),~
        a_{1,n}^{\prime}\rightarrow\theta_t^{\prime}a_{1,n}^{\prime\prime}~(1\leq n\leq n_1^\ast),~
        b_1\rightarrow \theta_t^{\prime} b_1^{\prime}\}$.
      \item[CT6]
        For $1\leq t\leq n_1^\ast$,
        $\{\theta_j^{\prime}\rightarrow a_{1,t}^{\prime}\theta_j^{\prime\prime}~(1\leq j \leq r),~
        a_{1,n}^{\prime}\rightarrow a_{1,t}^{\prime}a_{1,n}^{\prime\prime}~(1\leq n\leq n_1^\ast,\ n\neq t),~
        b_1\rightarrow a_{1,t}^{\prime} b_1^{\prime}\}$.
      \item[CT7]
        $\{\theta_j^{\prime}\rightarrow b_1\theta_j^{\prime\prime}~(1\leq j \leq r),~
        a_{1,n}^\prime\rightarrow b_1 a_{1,n}^{\prime\prime}~(1\leq n\leq n_1^\ast)\}$.
    \end{description}
    
    \subsubsection*{Apply CT5 once}
    Applying CT5,
    $
      \pi=
      \{\theta_j^{\prime}\rightarrow \theta_1^{\prime}\theta_j^{\prime\prime}~(2\leq j \leq r),~
      a_{1,n}^{\prime}\rightarrow \theta_1^\prime a_{1,n}^{\prime\prime}~(1\leq n \leq n_1^\ast),~
      b_1 \rightarrow \theta_1^\prime b_1^\prime\},
    $
    we obtain
    \begin{align*}
      f =& \theta_1^{\prime m_1+1}b_1^{\prime m_1}
      \Biggl\{
      D_1(a^\prime)
      +\sum_{i=2}^r\theta_i^{\prime\prime}D_i(a^\prime)
      + \sum_{n=1}^{n_1^\ast} a_{1,n}^{\prime\prime}Z_{1,n} \\
      &\quad\quad\quad\quad\quad
      +\sum_{s=2}^{L}\theta_1^{\prime m_s-m_1-1}b_1^{\prime m_s-m_1}\sum_{n=1}^{n_s^\ast}a_{s,n}^{\prime}Z_{s,n}
      + \theta_1^{\prime}\times(\cdots)
      \Biggr\}\\
       =& \theta_1^{\prime m_1+1}b_1^{\prime m_1}\tilde{f}(X)
       ~;\quad
       \left.\tilde{f}(X)\right|_{(\theta_1^{\prime},a_{2,\cdot}^{\prime},\ldots ,a_{L,\cdot}^{\prime})=0}
        =X_1+ \sum_{i=2}^r \theta_i^{\prime\prime}X_i+\sum_{n=1}^{n_1^\ast} a_{1,n}^{\prime\prime}Z_{1,n}
         \not\equiv 0
         \quad\text{in }L^2(q).
    \end{align*}
    Hence we obtain normal crossings.
    Since the Jacobian of this coordinate transformation is
    $\theta_1^{\prime (m_1+1)r+\beta+n_1^\ast-1}b_1^{\prime m_1r+\beta-1}$,
    we obtain
    \[
      \inf_{Q}\min_{j}{\frac{h_j^{(Q)}+1}{k_j^{(Q)}}}
      =\min\left\{\frac{(m_1+1)r+\beta+n_1^\ast}{2(m_1+1)},\frac{m_1r+\beta}{2m_1}\right\}.
    \]
    The same conclusion holds when CT6 is applied once.
    
    \subsubsection*{Apply CT7 $k$ times $(k\leq m_2-m_1-1)$, and then apply CT5 or CT6}
    First, after applying CT7 successively $k$ times, we may write
    $
      \pi=\{\theta_j^{\prime}\rightarrow b_1^k\theta_j^{\prime\prime}~(1\leq j \leq r),~
      a_{1,n}^{\prime} \rightarrow b_1^{k}a_{1,n}^{\prime\prime}~(1\leq n\leq n_1^\ast)\},
    $
    and obtain
    \begin{align*}
      f =& b_1^{m_1+k}
      \left\{
        \sum_{i=1}^r \theta_i^{\prime\prime} D_i(a^\prime)
        + \sum_{n=1}^{n_1^\ast} a_{1,n}^{\prime\prime}Z_{1,n}
        + \sum_{s=2}^L b_1^{m_s-m_1-k}\sum_{n=1}^{n_s^\ast}a_{s,n}^{\prime} Z_{s,n}
        + b_1\times(\cdots)
      \right\}\\
      =& b_1^{m_1+k}\tilde{f}(X)
        ~; \quad
        \left.\tilde{f}(X)\right|_{(b_1,a_{2,\cdot}^\prime,\ldots,a_{L,\cdot}^\prime)=0}
        =\sum_{i=1}^r \theta_i^{\prime\prime}X_i
        + \sum_{n=1}^{n_1^\ast} a_{1,n}^{\prime\prime}Z_{1,n}.
    \end{align*}
    By assumption~(iii), normal crossings are obtained at any point $Q$ satisfying
    $
      (\theta^{\prime\prime}_1,\ldots,\theta^{\prime\prime}_{r},a_{1,1}^{\prime\prime},\ldots,a_{1,n_1^\ast}^{\prime\prime})\neq 0.
    $
    Since the Jacobian of this coordinate transformation is
    $b_1^{(m_1+k)r+\beta+k n_1^\ast-1}$, we obtain
    \[
      \inf_{Q}\min_{j}{\frac{h_j^{(Q)}+1}{k_j^{(Q)}}}
      =\frac{(m_1+k)r+\beta+k n_1^\ast}{2(m_1+k)}.
    \]
    Therefore, it remains to consider only points satisfying
    $
      (\theta^{\prime\prime}_1,\ldots,\theta^{\prime\prime}_{r},a_{1,1}^{\prime\prime},\ldots,a_{1,n_1^\ast}^{\prime\prime})= 0.
    $
    We now denote the transformed coordinates
    $\theta_1^{\prime\prime},\ldots,\theta_r^{\prime\prime},a_{1,1}^{\prime\prime},\ldots,a_{1,n_1^\ast}^{\prime\prime}$
    again by
    $\theta_1^{\prime},\ldots,\theta_r^{\prime},a_{1,1}^{\prime},\ldots,a_{1,n_1^\ast}^{\prime}$.
    Applying CT5,
    $
      \{\theta_j^{\prime}\rightarrow \theta_1^{\prime}\theta_j^{\prime\prime}~(2\leq j \leq r),~
      a_{1,n}^{\prime}\rightarrow \theta_1^{\prime}a_{1,n}^{\prime\prime}~(1\leq n \leq n_1^\ast),~
      b_1\rightarrow \theta_1^{\prime}b_1^{\prime}\},
    $
    we obtain
    \begin{align*}
      f =& \theta_1^{\prime m_1+k+1}b_1^{\prime m_1+k}
      \Biggl\{
        D_1(a^\prime)
        + \sum_{i=2}^r \theta_i^{\prime\prime}D_i(a^\prime)
        + \sum_{n=1}^{n_1^\ast} a_{1,n}^{\prime\prime}Z_{1,n}\\
        & \quad\quad\quad\quad\quad\quad\quad\quad
        + \sum_{s=2}^L \theta_1^{\prime m_s-m_1-k-1}b_1^{\prime m_s-m_1-k}\sum_{n=1}^{n_s^\ast}a_{s,n}^\prime Z_{s,n}
        + \theta_1^{\prime}\times(\cdots)
      \Biggr\}\\
      =& \theta_1^{\prime m_1+k+1}b_1^{\prime m_1+k}\tilde{f}(X)
      ~;\\
      \left.\tilde{f}(X)\right|&_{(\theta_1^{\prime},a_{2,\cdot}^{\prime},\ldots,a_{L,\cdot}^{\prime})=0}
      = X_1 + \sum_{i=2}^r \theta_i^{\prime\prime}X_i
      + \sum_{n=1}^{n_1^\ast} a_{1,n}^{\prime\prime}Z_{1,n}
      \not\equiv0
      \quad\text{in }L^2(q).
    \end{align*}
    Hence, normal crossings are obtained.
    Since the Jacobian of this coordinate transformation is
    $\theta_1^{\prime (m_1+k+1)r+\beta+(k+1)n_1^\ast-1}b_1^{\prime (m_1+k)r+\beta+k n_1^\ast-1}$,
    we obtain
    \[
      \inf_{Q}
      \min_{j}{\frac{h_j^{(Q)}+1}{k_j^{(Q)}}}
      =\min\left\{
        \frac{(m_1+k+1)r+\beta+(k+1)n_1^\ast}{2(m_1+k+1)},
        \frac{(m_1+k)r+\beta+k n_1^\ast}{2(m_1+k)}
      \right\}.
    \]
    As a function of $k$, the quantity
    $
      \frac{(m_1+k)r+\beta+k n_1^\ast}{2(m_1+k)}
    $
    is either monotone increasing or monotone decreasing, and hence its minimum is attained at either $k=0$ or $k=m_2-m_1-1$.
    The necessary and sufficient condition for these two values to coincide is $\beta=m_1 n_1^\ast$.
    
    The above argument was given for CT5, but the same conclusion holds when CT6 is applied instead.
    
    \subsubsection*{Apply CT7 $m_2-m_1$ times}
    Proceeding as above, we apply the coordinate transformation
    $
      \pi=\{\theta_j^{\prime}\rightarrow b_1^{m_2-m_1}\theta_j^{\prime\prime}~(1\leq j \leq r),~
      a_{1,n}^{\prime} \rightarrow b_1^{m_2-m_1}a_{1,n}^{\prime\prime}~(1\leq n\leq n_1^\ast)\},
    $
    and obtain
    \begin{equation}\label{eq_pf_mainthm_step3-1}
      \begin{aligned}
        f =& b_1^{m_2}
        \Biggl\{
          \sum_{i=1}^r \theta_i^{\prime\prime} D_i(a^\prime)
          + \sum_{n=1}^{n_1^\ast} a_{1,n}^{\prime\prime}Z_{1,n}
          + \sum_{n=1}^{n_2^\ast} a_{2,n}^{\prime}Z_{2,n}
          + \sum_{s=3}^L b_1^{m_s-m_2}\sum_{n=1}^{n_s^\ast}a_{s,n}^\prime Z_{s,n}
          \\
        &\quad\quad
          + b_1^{m_L-m_2}\sum_{(s,n)\in\mathcal{S}} k_{s,n}^{(2)}Z_{s,n}
          + b_1^{m_2} \sum_{s=1}^\infty \tilde{h}^{(2)}_s W_s
          + \sum_{\substack{j+l\geq 2\\j\geq1,l\geq 0}} b_1^{m_2(j-1)+l} h_{j,l}^{(2)}(\theta^{\prime\prime})W_{j,l}^\prime
        \Biggr\}\\
        =& b_1^{m_2}\tilde{f}(X)
        ~;\quad
        \left.\tilde{f}(X)\right|_{(b_1,a_{2,\cdot}^{\prime},\ldots,a_{L,\cdot}^{\prime})=0}
        = \sum_{i=1}^r \theta_i^{\prime\prime}X_i
        + \sum_{n=1}^{n_1^\ast} a_{1,n}^{\prime\prime}Z_{1,n}.
      \end{aligned}
    \end{equation}
    In this coordinate neighborhood, normal crossings are obtained at any point $Q$ satisfying
    $
      (\theta_1^{\prime\prime},\ldots,\theta_r^{\prime\prime},a_{1,1}^{\prime\prime},\ldots,a_{1,n_1^\ast}^{\prime\prime})\neq0.
    $
    Since the Jacobian of this coordinate transformation is
    $b_1^{m_2r+\beta+(m_2-m_1)n_1^\ast-1}$, we obtain
    \[
      \inf_{Q}
      \min_{j}{\frac{h_j^{(Q)}+1}{k_j^{(Q)}}}
      =\frac{m_2r+\beta+(m_2-m_1)n_1^\ast}{2m_2}.
    \]
    Therefore, it remains to consider only points satisfying
    $
      (\theta_1^{\prime\prime},\ldots,\theta_r^{\prime\prime},a_{1,1}^{\prime\prime},\ldots,a_{1,n_1^\ast}^{\prime\prime})=0.
    $
    This case will be considered in the next step.
    
    Summarizing, among the normal crossings obtained in Step~3-1, we have
    \[
      \inf_{Q}
      \min_{j}{\frac{h_j^{(Q)}+1}{k_j^{(Q)}}}
      =\min\left\{
      \frac{m_1r+\beta}{2m_1},
      \frac{m_2r+\beta+(m_2-m_1)n_1^\ast}{2m_2}
      \right\}.
    \]
    We continue to denote the transformed coordinates
    $
      \theta_1^{\prime\prime},\ldots,\theta_{r}^{\prime\prime},~
      a_{1,1}^{\prime\prime},\ldots,a_{1,n_1^\ast}^{\prime\prime}
    $
    by
    $
      \theta_1^{\prime},\ldots,\theta_{r}^{\prime},~
      a_{1,1}^{\prime},\ldots,a_{1,n_1^\ast}^{\prime}.
    $

  \subsubsection*{Step 3-$k$: Coordinate transformations for $k=2,3,\ldots,L-1$}
    For each $k=2,3,\ldots,L-1$, we perform the following three types of coordinate transformations, CT5$(k)$--CT7$(k)$,
    \footnote{In the terminology of algebraic geometry, these coordinate transformations are described as the blow-up centered at the subvariety
    $\{(\theta^{\prime},a^{\prime},b_1,\tau)\mid \theta^{\prime}=a_{1,\cdot}^{\prime}=\cdots=a_{k,\cdot}^{\prime}=b_1=0\}$.}
    according to the procedure in Figure~\ref{fig3}.
    \begin{description}
      \item[CT5$(k)$]
        For $1\leq t\leq r$,
        $\theta_j^{\prime}\rightarrow \theta_t^{\prime}\theta_j^{\prime\prime}$ $(1\leq j \leq r,\ j\neq t)$,
        $a_{s,n}^{\prime}\rightarrow\theta_t^{\prime}a_{s,n}^{\prime\prime}$ $(1\leq s\leq k,\ 1\leq n\leq n_s^\ast)$,
        and $b_1\rightarrow \theta_t^\prime b_1^{\prime}$.
      \item[CT6$(k)$-$i~(1\leq i\leq k)$]
        For $1\leq t\leq n_i^\ast$,
        $\theta_j^{\prime}\rightarrow a_{i,t}^{\prime}\theta_j^{\prime\prime}$ $(1\leq j \leq r)$,
        $a_{s,n}^{\prime}\rightarrow a_{i,t}^{\prime}a_{s,n}^{\prime\prime}$ $(1\leq s\leq k,\ s\neq i,\ 1\leq n\leq n_s^\ast)$,
        $a_{i,n}^{\prime}\rightarrow a_{i,t}^{\prime}a_{i,n}^{\prime\prime}$ $(1\leq n\leq n_i^\ast,\ n\neq t)$,
        and $b_1\rightarrow a_{i,t}^{\prime} b_1^{\prime}$.
      \item[CT7$(k)$]
        $\theta_j^{\prime}\rightarrow b_1\theta_j^{\prime\prime}$ $(1\leq j \leq r)$
        and $a_{s,n}^{\prime}\rightarrow b_1a_{s,n}^{\prime\prime}$ $(1\leq s\leq k,\ 1\leq n\leq n_s^\ast)$.
    \end{description}

    Note that CT5(1)--CT7(1) coincide with CT5--CT7 in Step~3-1.
    Proceeding as in the case $k=1$, we find that, in the coordinate neighborhoods obtained after carrying out the coordinate transformations up to Step~3-$k$,
    \begin{align*}
      &\inf_{Q}
      \min_{j}{\frac{h_j^{(Q)}+1}{k_j^{(Q)}}}\\
      =&\min\left\{
      \frac{m_kr+\beta+\sum_{s=1}^{k-1}(m_{k}-m_s)n_s^\ast}{2m_k},
      \frac{m_{k+1}r+\beta+\sum_{s=1}^k(m_{k+1}-m_s)n_s^\ast}{2m_{k+1}}
      \right\}.
    \end{align*}
    A necessary and sufficient condition for these two quantities to be equal is
    $\beta=\sum_{s=1}^k m_s n_s^\ast$.
    
    Proceeding inductively on $k$ in the same manner as in Step~3-1, we find that, at each stage $k=2,\ldots,L-1$, 
    the only coordinate neighborhood in which normal crossings have not yet been obtained is the one arising after applying CT7$(k)$ exactly $m_{k+1}-m_k$ times.
    After completing these steps and relabeling the transformed coordinates
    $
      \theta_1^{\prime\prime},\ldots,\theta_{r}^{\prime\prime},\ a_{s,n}^{\prime\prime}
    $
    as
    $
      \theta_1^{\prime},\ldots,\theta_{r}^{\prime},\
      a_{s,n}^\prime
      \quad (1\leq s\leq L-1,\ 1\leq n\leq n_s^\ast),
    $
    we can express $f$ as follows:
    \begin{equation}\label{eq_pf_mainthm_step3-k}
      \begin{aligned}
        f =& b_1^{m_L}
        \Biggl\{
          \sum_{i=1}^{r} \theta_i^{\prime} D_i(a^\prime)
          + \sum_{s=1}^{L}\sum_{n=1}^{n_s^\ast} a_{s,n}^{\prime}Z_{s,n}
          + \sum_{(s,n)\in \mathcal{S}} k_{s,n}^{(L)}Z_{s,n}
          \\
        &\quad\quad
          + b_1^{m_L}\sum_{s=1}^\infty \tilde{h}^{(L)}_s W_s
          + \sum_{\substack{j+l\geq 2\\j\geq1,l\geq0}}b_1^{m_L(j-1)+l} h_{j,l}^{(L)}(\theta^\prime)W_{j,l}^{\prime\prime}
        \Biggr\}.
      \end{aligned}
    \end{equation}
    Here,
    $
      k_{s,n}^{(L)}\in \text{ideal }
      \left(b_1^{m_L-m_s}a_{s,n}^\prime \mid 1\leq s\leq L,\ 1\leq n\leq n_s^\ast\right),
    $
    and
    $
      \tilde{h}^{(L)}_s\in \text{ideal }
      (a_{s,n}^\prime \mid 1\leq s\leq L,\ 1\leq n\leq n_s^\ast)^2.
    $
    Moreover, $h_{j,l}^{(L)}(\theta^\prime)$ remains homogeneous of degree $j$ in $\theta^\prime$.

  \subsection*{Step 4: Blow-up centered at $(\theta^\prime,a_{1,\cdot}^\prime,\ldots,a_{L,\cdot}^\prime)=0$}
    We perform each of the following two types of coordinate transformations, CT8 and CT9, once.
    \begin{description}
      \item[CT8]
        For $1\leq t\leq r$,
        $\theta_j^{\prime}\rightarrow \theta_t^{\prime}\theta_j^{\prime\prime}$ $(1\leq j \leq r,\ j\neq t)$
        and $a_{s,n}^{\prime}\rightarrow \theta_t^{\prime}a_{s,n}^{\prime\prime}$ $(1\leq s\leq L,\ 1\leq n\leq n_s^\ast)$.
      \item[CT9-$i~(1\leq i \leq L)$]
        For $1\leq t\leq n_i^\ast$,
        $\theta_j^{\prime}\rightarrow a_{i,t}^{\prime}\theta_j^{\prime\prime}$ $(1\leq j \leq r)$,
        $a_{s,n}^{\prime}\rightarrow a_{i,t}^{\prime}a_{s,n}^{\prime\prime}$ $(1\leq s\leq L,\ s\neq i,\ 1\leq n\leq n_s^\ast)$,
        and $a_{i,n}^{\prime}\rightarrow a_{i,t}^{\prime}a_{i,n}^{\prime\prime}$ $(1\leq n\leq n_i^\ast,\ n\neq t)$.
    \end{description}   
  
    After applying CT8 once, (\ref{eq_pf_mainthm_step3-k}) becomes
    \begin{align*}
      f =& b_1^{m_L}\theta^{\prime}_1
      \Biggl\{
        D_1(a^\prime)+
        \sum_{i=2}^{r} \theta_i^{\prime\prime} D_i(a^\prime)
        + \sum_{s=1}^L \sum_{n=1}^{n_s^\ast} a_{s,n}^{\prime\prime} Z_{s,n}
        + \sum_{n=n_L^\ast+1}^{n_L} k_{L,n}^\prime Z_{L,n}
        + b_1\times(\cdots)
      \Biggr\}\\
      =& b_1^{m_L}\theta_1^{\prime}\tilde{f}
      ~;\\
      \left.\tilde{f}(X)\right|&_{(b_1,\theta^{\prime}_1)=0}
      =
      X_1+
      \sum_{i=2}^{r} \theta_i^{\prime\prime} X_i
      + \sum_{s=1}^L \sum_{n=1}^{n_s^\ast} a_{s,n}^{\prime\prime} Z_{s,n}
      + \sum_{n=n_L^\ast+1}^{n_L} k_{L,n}^\prime Z_{L,n}
      \not\equiv0
      \quad\text{in }L^2(q).
    \end{align*}
    Here the terms indexed by $(s,n)\in\mathcal S$ with $s\geq L+1$ are absorbed into $b_1\times(\cdots)$, since $m_s-m_L\geq 1$ for such $s$.
    Hence, normal crossings are obtained.
    Since the Jacobian of this coordinate transformation is
    $b_1^{m_Lr+\beta+\sum_{s=1}^{L-1}(m_L-m_s)n_s^\ast-1}\theta_1^{\prime r+\sum_{s=1}^L n_s^\ast-1}$,
    we obtain
    \[
      \inf_{Q}
      \min_{j}{\frac{h_j^{(Q)}+1}{k_j^{(Q)}}}
      =\min\left\{
      \frac{m_Lr+\beta+\sum_{s=1}^{L-1}(m_L-m_s)n_s^\ast}{2m_L}, 
      \frac{r+\sum_{s=1}^L n_s^\ast}{2}
      \right\}.
    \]
    The case of CT9-$i$ is analogous.  
    Therefore, we obtain \eqref{eq_pf_mainthm00}.
    This completes the proof.

\end{document}